%% file: main.tex
\documentclass{opt2026} 

\title[Curvature-Aware Adaptive Perturbation SAM]{One Size Does Not Fit All: Curvature-Aware Adaptive Perturbation for SAM}

\usepackage{tikz}
\usepackage{multirow}

\usepackage{pgfplots}
\usepackage{subcaption}
\usepackage{algorithm}
\usepackage{algorithmic}
\usepackage{wrapfig}
\usepackage{amsmath,amssymb}
\usepackage{mathtools}
\usetikzlibrary{calc,arrows.meta}   
\usepackage{booktabs}               

\newtheorem{assumption}{Assumption}

\optauthor{%
\Name{Alessia Rossi} \Email{alessia.rossi@bristol.ac.uk}\\
\Name{Felipe Campelo} \Email{f.campelo@bristol.ac.uk}\\
\Name{Gabriel L. Oliveira} \Email{gabriel.leivasoliveira@bristol.ac.uk}\\
\addr School of Engineering Mathematics and Technology, University of Bristol, Bristol, Lovelace Building, Ada, Bristol, BS8 1TW, UK}

\begin{document}

\maketitle

\begin{abstract}%
Poor generalisation in deep neural networks is frequently associated with convergence to sharp minima. Sharpness-Aware Minimisation (SAM) improves generalisation of deep neural networks by minimising a combination of the training loss and the local sharpness of the loss landscape, biasing convergence towards flat minima. However, SAM relies on a fixed perturbation radius $\rho$ that ignores local curvature, leading to inefficient behaviour: the perturbation may be too small to escape poor basins defined by sharp local minima, or too large to allow convergence to solutions with good loss-sharpness tradeoff. We reinterpret $\rho$ as the intensity of the perturbation signal and propose Curvature-Aware Adaptive Perturbation for SAM (CAP-SAM), a geometry-aware method that modulates the perturbation radius based on a sharpness ratio, enabling a more effective search for robust minima on the loss landscape. We prove that CAP-SAM preserves the same convergence rate as standard SAM in non-convex stochastic settings. Extensive experiments indicate that our approach consistently outperforms SAM and its most prominent variants across diverse network architectures on CIFAR-10, CIFAR-100, and ImageNet.
\end{abstract}

\section{Introduction}

Despite the remarkable success of deep neural networks in a wide range of tasks \cite{he2016deep,Samuel59,krizhevsky2009learning,vaswani2017attention}, their ability to generalise to unseen data remains a persistent challenge in modern deep learning \cite{DBLP:journals/corr/ZhangBHRV16,hendrycks2021many}.
Recent work suggests that the geometry of the loss landscape plays a central role in model generalisation \cite{2998687.2998753,Measures,DBLP:journals/corr/NeyshaburBMS17}.
Flat minima have long been associated with better generalisation \cite{DBLP:journals/corr/DinhPBB17,andriushchenko2023modern,6796962}, and it has been observed that sharp minima, regions of high curvature where small perturbations cause large increases in loss, generalise poorly to unseen data \cite{chaudhari_entropy-sgd_2017} .

This led to \textit{Sharpness-Aware Minimization} (SAM) \cite{sam} which minimises the maximum loss within a neighbourhood of the current parameters, enabling it to identify flat minima associated with improved generalisation performance. When the model reaches a sharp region, the perturbation step 
helps it escape the sharp minimum and move toward flatter areas of the loss landscape \cite{DBLP:journals/corr/KeskarMNST16}.

The perturbation radius $\rho$ is among SAM's most critical hyperparameters: its value strongly affects the solution quality, and it must be carefully tuned for each dataset and architecture. A key limitation of current formulations is that $\rho$ is fixed throughout training, so a single value must serve regions of different curvature. It may be insufficient to escape very sharp regions, leaving the optimizer trapped in poor minima, or perturb too aggressively in flat regions, pushing the optimizer away from a good solution. This motivates a new treatment of $\rho$; rather than a static hyperparameter, we approach it as a dynamic quantity that should adapt to the local geometry of the loss landscape.


The contributions of this paper can be summarized as follows:
\begin{itemize}
    \item We provide a new interpretation of the perturbation radius $\rho$ in SAM, revealing its limitations, and proposing CAP-SAM, which adapts its radius to local sharpness. 
    \item We prove that CAP-SAM converges at the same rate as SAM, i.e., $\mathcal{O}\!\left(\log T / \sqrt{T}\right)$. We further show that SAM's perturbation becomes geometry-blind near stationary points and that our approach mitigates this.
    \item We validate CAP-SAM on several baselines and datasets (CIFAR-10/100 and ImageNet), demonstrating improvements in accuracy and generalisation over SAM and its variants.
\end{itemize}

\begin{figure}[h]
    \centering
    \includegraphics[width=0.75\linewidth]{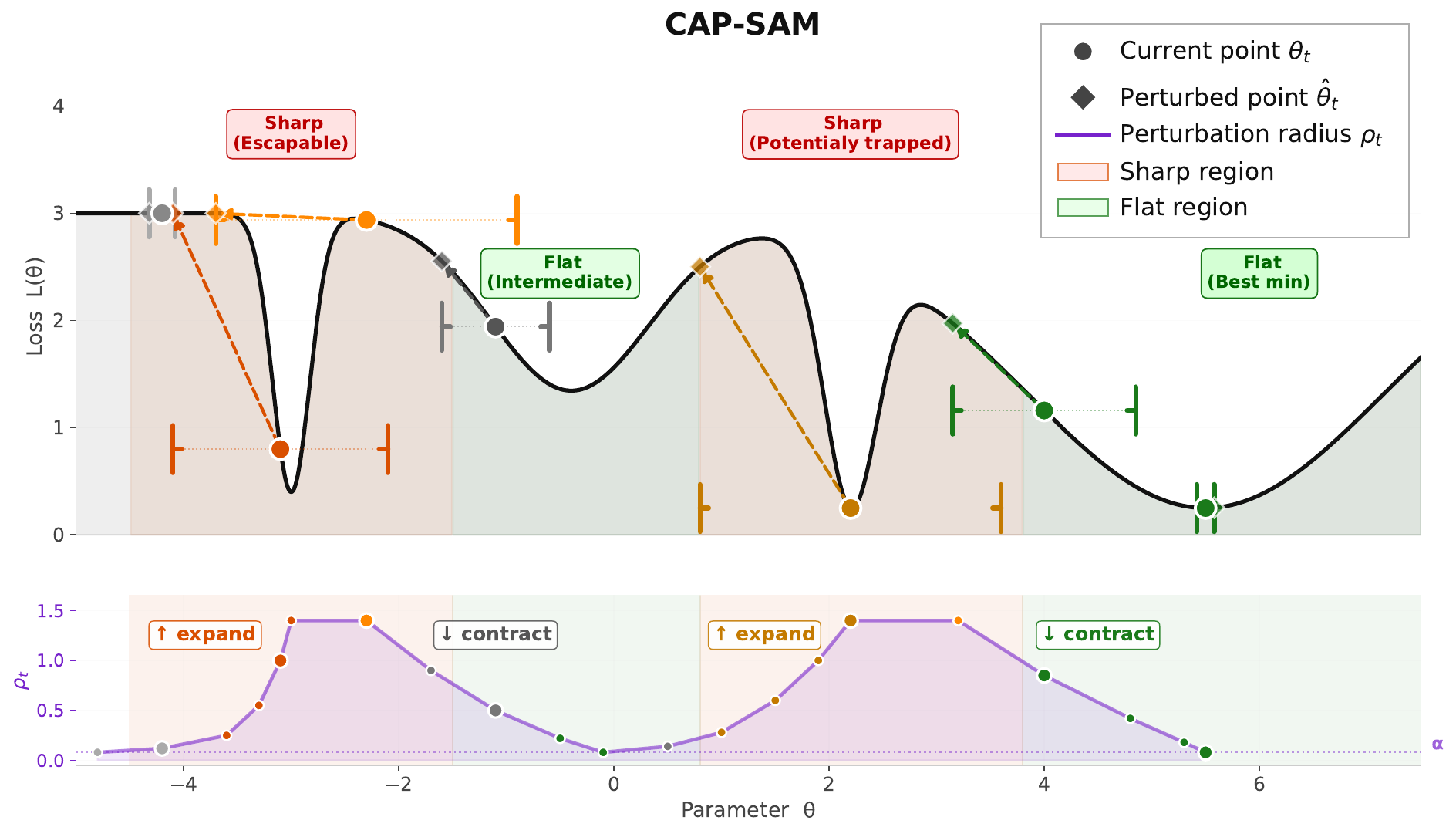}
    \caption{
    Curvature-aware perturbation radius scheduling. The top panel illustrates how our method adjusts the perturbation radius based on the sharpness characteristics of the loss landscape.
    The bottom panel shows how $\rho_t$ evolves: when in a sharp region, $\rho_t$ expands by a factor of $\gamma$ to look further ahead and escape the sharp basin; when in a flat region, $\rho_t$ contracts by the same factor, so the optimizer stops perturbing aggressively and settles into the local minimum.}

    
    \label{fig:sam_adaptive}
\end{figure} 
 \vspace{-15pt}

\section{Methodology}


We introduce a dynamically adaptive perturbation radius that addresses the fundamental limitations of the standard SAM formulation, based on a novel reinterpretation of the perturbation radius $\rho$ as the magnitude of the perturbation signal applied to the parameter vector. Vanilla SAM \cite{sam} employs a fixed perturbation $\rho$ which remains constant throughout the optimization procedure. In contrast, our formulation modulates this intensity at each step in response to the local sharpness changes observed during training.
The method is seamlessly incorporated into any SAM variant, augmenting its capabilities at negligible additional computational cost.



\subsection{Sharpness-Aware Minimization}
SAM is an optimization approach focusing on robust learning to enhance the generalisation ability of deep neural networks 
by solving the min-max problem


\[
  \min_w \; f^{\mathrm{SAM}}(w)  
  \quad\text{where}\quad 
  f^{\mathrm{SAM}}(w) \triangleq \max_{\|\epsilon\| \leq \rho}f(w + \epsilon);
\]

\noindent where $\epsilon$ denotes the perturbation applied to the weights
$w$, $\rho$ is the fixed perturbation radius that controls the size of the adversarial neighborhood.

The perturbation at iteration $t$ is computed as 
\begin{align}
\epsilon_t = \arg\max_{\|\epsilon\| \leq \rho} f(w_t + \epsilon) 
\approx \arg\max_{\|\epsilon\|\leq \rho} f(w_t) + \epsilon\nabla f(w_t) 
\approx \rho \frac{g_t(w_t)}{\|g_t(w_t)\|}, 
\label{eq:sam_perturbation}
\end{align}

SAM minimises the local worst-case loss by perturbing the parameter space, encouraging  convergence to flat minima. 





\noindent\textbf{Interpretation of SAM. } Notice that SAM formulates the perturbation as $\epsilon_t= \rho\frac{g_t(w_t)}{\|g_t(w_t)\|}$ \cite{sam} where $\rho $ denotes the radius of the perturbation neighborhood and $\frac{g_t(w_t)}{\|g_t(w_t)\|}$ is the direction of steepest ascent. However, we observe that $\rho $  can equivalently be interpreted as a multiplicative scalar controlling the magnitude of the ascent step, i.e., how far the parameters are displaced along the steepest ascent direction.
This perspective exposes a formal limitation of using a fixed radius, which constrains the perturbation to a fixed magnitude regardless of the local sharpness of the loss landscape. In very sharp regions, a fixed $\rho $  may be insufficient to escape the basin, while in flat regions the same $\rho $ may applies an unnecessary perturbation.
 
\subsection{Dynamic Radius Formulation}
The proposed method modifies the ascent step of SAM by introducing a dynamically adaptive perturbation radius $\rho_t$, adjusted at each optimization step based on local sharpness changes observed during training. The core of our strategy is centred around the sharpness ratio, defined as 
\begin{align}
r_t= 
\frac{\left| f(w_t + \epsilon_{w_t}) - f(w_t) \right|}
{\left| f(w_{t-1} + \epsilon_{w_{t-1}}) - f(w_{t-1}) \right|}
\approx 
\frac{\left| \langle g(w_t), \epsilon_{w_t} \rangle \right|}
{\left| \langle g(w_{t-1}), \epsilon_{w_{t-1}} \rangle \right|}
\approx 
\frac{\|g(w_t)\|}{\|g(w_{t-1})\|}.
\label{ratio}
\end{align}

\noindent capturing the relative change in local sharpness between consecutive optimization steps (Appendix \ref{r_approximation} for a theoretical explanation).
The perturbation radius is then updated according to 
\begin{align}
\rho_t =
      \begin{cases}
        \min\!\left(\gamma\,\rho_{t-1},\;\rho_{\max}\right),              & r_t \geq 1 \\[4pt]
        \max\!\left(\dfrac{\rho_{t-1}}{\gamma},\;\rho_{\min}\right), & \text{otherwise}
      \end{cases}
\label{rho}
\end{align}

When $r_t \geq 1$ the optimization has moved into a sharper region of the loss surface and $\rho_t$ grows geometrically by a factor of $\gamma$, until it is large enough to expose the optimizer to flatter regions.  On the other hand, when $r_t \leq 1$ the optimizer has entered a flatter region and $\rho_t$ contracts by the same factor, preserving the quality of the minimum found and preventing over perturbation.
The bounds $\rho_{\min}$
and $\rho_{\max}$ constrains the range of $\rho_t$ to $[\rho_{\min},\rho_{\max}]$.

Figure \ref{fig:sam_adaptive} depicts  the resulting radius dynamics across different regions of the loss surface. Our method treats $\rho_t$ as a quantity that stretches and shrinks in response to the local geometry of the loss landscape. 

   

 
\section{ Theoretical Analysis}

\subsection{Convergence analysis}
In this section, we analyze the convergence properties of CAP-SAM in nonconvex settings. 
Based on standard assumptions in nonconvex stochastic optimization (stated formally in Appendix \ref{app:assumptions}), we can now define Theorem \ref{thm:convergence}.

\begin{theorem}
\label{thm:convergence}
Suppose that Assumptions~\ref{ass:bounded_gradient}, \ref{ass:bounded_variance},
and \ref{ass:smoothness} hold, and that the perturbation radius satisfies the \ref{decaying_radius}, e.g., $\rho_t = \rho_{\max}/\sqrt{t}$ (Appendix \ref{app:assumptions}). Let $\eta_t = \eta_0/\sqrt{t}$. 
For $\rho_{\max} \leq \eta_0G$ and $\eta_0 \leq 1~/~L$, we have:
\[
    \frac{1}{T}\sum_{t=1}^{T}
    \mathbb{E}[\|\nabla f(w_t)\|^2]
    \;\leq\;
    \frac{C_1+C_2\log T}{\sqrt{T}},
\]
where
$
    C_1 \;=\; \frac{2}{\eta_0}
    \left(f(w_0) - \mathbb{E}\left[f(w_T)\right]\right),
    \qquad
    C_2 \;=\; 2L\left(\sigma^2\eta_0 
    + \rho_{\max}G\right).
$
\end{theorem}
For non-convex stochastic optimization, Theorem \ref{thm:convergence} implies that  the dynamic radius formulation has the same convergence rate as SAM, i.e., $\mathcal{O}\!\left(\log T / \sqrt{T}\right)$. 
Detailed proofs are in Appendix \ref{theorem}. 

\subsection{CAP-SAM reduces noise impact around the optimum}
\input{rho_impact_misalignment}
Given the stochastic gradient decomposition $g_t(w_t) = \nabla f(w_t) + \xi$, 
in the near-optimal regime where $\|\nabla f(w_t)\| \to 0$, the gradient 
becomes dominated by noise, i.e., $g_t(w_t) \approx \xi$. 
The SAM perturbation 
thus collapses to:
$\epsilon_t \approx \rho~\xi ~/~\|\xi\|$.
For Gaussian noise, $\epsilon_t $ is uniformly distributed on the sphere $\mathcal{S}_\rho(0)$, no longer reflecting directions of high curvature or sharpness, affecting the ability of SAM to accurately estimate local flatness.
The dynamic $\rho_t$ mitigates the effect of directional misalignment. By construction, $\rho_t$ tends to increase in sharp regions where $\|\nabla f(w_t)\|$ is large, and reduce in flat regions where $\|\nabla f(w_t)\| \approx 0$, naturally scaling the magnitude of the perturbation to the local geometry, limiting the influence of 
non-informative, noise-driven perturbations, as illustrated in Figure \ref{fig:misalignment}. 
A formal geometric analysis of the gradient misalignment and its implications are provided in Appendix \ref{gradientmismisalignment}.

\section{ Experimental Results}
\label{experiments}
We evaluate the proposed CAP-SAM through extensive experiments on standard benchmark datasets across diverse model architectures. We compare each method in two configurations: using a fixed perturbation radius $\rho$ and using our proposed dynamic $\rho$. 

\noindent\textbf{Experimental Settings.} We adopt CIFAR-10 \cite{krizhevsky2009learning}, CIFAR-100 \cite{krizhevsky2009learning} and ImageNet-1K \cite{imagenet_deng09} as benchmark datasets to evaluate the proposed method. For CIFAR-10 and CIFAR-100, we compare against SAM and three of its variants: ASAM \cite{kwon2021asam}, Fisher SAM \cite{kim22f}, and VASSO \cite{li2023enhancing}, across four neural network architectures: VGG-11 \cite{vgg}, ResNet-18 \cite{he2016deep}, WRN-28-10 \cite{DBLP:journals/corr/ZagoruykoK16} and PyramidNet-110 \cite{DBLP:journals/corr/HanKK16}.
For ImageNet-1K \cite{imagenet_deng09}, we train ResNet-50 from scratch and evaluate on the ImageNet validation set as well as four additional out-of-distribution test sets: ImageNet-Real/V2/R/Sketch~\cite{beyer2020we,recht2019imagenet,hendrycks2021many,wang2019learning}. We evaluate the proposed dynamic formulation applied to SAM, ASAM \cite{kwon2021asam}, and GSAM \cite{zhuang2022surrogate}.  
A comprehensive description of the experimental setup, including hyperparameters and a detailed description of the statistical tests, is provided in the Appendix \ref{Exp_setup}.

\noindent\textbf{Results.}
Tables \ref{tab:cifar10_adaptive_vs_constant} and \ref{tab:cifar100_adaptive_vs_constant} report the results of our method across CIFAR-10 and CIFAR-100, respectively. For CIFAR-10 our approach improves over the fixed radius baselines in $15$ of $16$ combinations. Gains are consistently more prominent for CIFAR-100 (Table \ref{tab:cifar100_adaptive_vs_constant}), where even for VaSSO, the strongest fixed-radius baseline on PyramidNet-110, we outperform the baselines. The mean gains observed for the dynamic $\rho$ over the constant radius are statistically significant at the $95\%$ level (CIFAR10: $p=0.0002,~CI_{95}=0.23\pm 0.11$; CIFAR100: $p=0.0013,~CI_{95}=1.14\pm 0.64$).


Table \ref{tab:imagenet_resnet50} reports results on ImageNet-1K and four out-of-distribution test sets. On ImageNet we significantly outperform the fixed-radius baselines ($p = 0.0012,~CI_{95}=0.31\pm 0.17$), with positive gains observed in $14$ of $15$ combinations, the sole exception being SAM on ImageNet-Sketch. All configurations surpass the vanilla SGD baseline (MHT-corrected $p<10^{-4}$ in all cases), confirming that curvature-aware adaptation is additive to the general benefit of sharpness-aware training rather than a substitute for it.

\noindent\textbf{Generalisation Effect.} Figure \ref{fig:pyramidnet110-cifar100} shows that CAP-SAM achieves a lower training loss than SAM \input{convergence} \noindent while reaching better test accuracy than both baselines. SGD prioritises data fitting at the cost of overfitting, SAM prioritises robustness but may underfit, and CAP-SAM balances the two, fitting the data effectively without sacrificing generalisation, converging  to flatter minima with stronger generalisation properties, a pattern consistent with \cite{DBLP:journals/corr/KeskarMNST16, DBLP:journals/corr/NeyshaburBMS17}. 
CAP-SAM also exhibits faster convergence than SAM and SGD (Figure \ref{fig:pyramidnet110-cifar100} (b)), suggesting a more efficient progress along the optimization trajectory.

\begin{table} [h]
\centering
\setlength{\tabcolsep}{5pt}
\resizebox{\textwidth}{!}{%
\begin{tabular}{l cc cc cc cc}
\toprule
& \multicolumn{2}{c}{\textbf{SAM}}
& \multicolumn{2}{c}{\textbf{ASAM}}
& \multicolumn{2}{c}{\textbf{Fisher SAM}}
& \multicolumn{2}{c}{\textbf{VaSSO}} \\
\cmidrule(lr){2-3}\cmidrule(lr){4-5}
\cmidrule(lr){6-7}\cmidrule(lr){8-9}
\textbf{Model}
  & Const. & Dynam.
  & Const. & Dynam.
  & Const. & Dynam.
  & Const. & Dynam. \\
\midrule
VGG-11-BN
  & $93.82_{\pm0.05}$ & $\mathbf{94.14_{\pm0.09}}$
  & $93.47_{\pm0.04}$ & $\mathbf{94.13_{\pm0.27}}$
  & $93.60_{\pm0.09}$ & $\mathbf{94.14_{\pm0.15}}$
  & $93.93_{\pm0.13}$ & $\mathbf{94.17_{\pm0.14}}$ \\
ResNet-18
  & $96.58_{\pm0.10}$ & $\mathbf{96.66_{\pm0.03}}$
  & $96.33_{\pm0.09}$ & $\mathbf{96.91_{\pm0.08}}$
  & $96.62_{\pm0.07}$ & $\mathbf{96.79_{\pm0.06}}$
  & $96.60_{\pm0.08}$ & $\mathbf{96.74_{\pm0.23}}$ \\
WRN-28-10
  & $97.32_{\pm0.11}$ & $\mathbf{97.47_{\pm0.06}}$
  & $97.15_{\pm0.05}$ & $\mathbf{97.56_{\pm0.07}}$
  & $97.46_{\pm0.06}$ & $\mathbf{97.48_{\pm0.10}}$
  & $97.46_{\pm0.09}$ & $97.16_{\pm0.08}$ \\
PyramidNet-110
  & $97.58_{\pm0.02}$ & $\mathbf{97.70_{\pm0.05}}$
  & $97.56_{\pm0.11}$ & $\mathbf{97.76_{\pm0.03}}$
  & $97.59_{\pm0.07}$ & $\mathbf{97.77_{\pm0.06}}$
  & $97.64_{\pm0.09}$ & $\mathbf{97.76_{\pm0.05}}$ \\
\midrule
\end{tabular}}
\caption{%
  Top-1 test accuracy (\%) on CIFAR-10 with Cutout augmentation.
}
\label{tab:cifar10_adaptive_vs_constant}
\end{table}
\newpage
\noindent\textbf{Hyperparameter sensitivity.} We analyze the impact of the hyperparameters of CAP-SAM.
\input{hyperparameter} \noindent As shown in Figure \ref{fig:ablation-hyperparams}, the parameters that meaningfully influence performance are the initial radius $\rho_0$ and the growth factor $\gamma$, which benefit from light tuning, consistent with standard practice for SAM and its variants. The lower and upper bounds, $\rho_{min}$ and  $\rho_{\max}$, have negligible impact on performance (Figures \ref{fig:ablation-hyperparamsmaxmin} (c) and (d) in Appendix \ref{sensitivity}), as they serve primarily to clip the updated radius within a safe interval. 

\begin{table}[t]
\centering
\setlength{\tabcolsep}{5pt}
\resizebox{\textwidth}{!}{%
\begin{tabular}{l cc cc cc cc}
\toprule
& \multicolumn{2}{c}{\textbf{SAM}}
& \multicolumn{2}{c}{\textbf{ASAM}}
& \multicolumn{2}{c}{\textbf{Fisher SAM}}
& \multicolumn{2}{c}{\textbf{VaSSO}} \\
\cmidrule(lr){2-3}\cmidrule(lr){4-5}
\cmidrule(lr){6-7}\cmidrule(lr){8-9}
\textbf{Model}
  & Const. & Dynam.
  & Const. & Dynam.
  & Const. & Dynam.
  & Const. & Dynam. \\
\midrule
VGG-11-BN
  & $73.59_{\pm0.10}$ & $\mathbf{74.15_{\pm0.14}}$
  & $68.53_{\pm0.15}$ & $\mathbf{73.46_{\pm0.09}}$
  & $73.82_{\pm0.31}$ & $\mathbf{74.38_{\pm0.13}}$
  & $74.02_{\pm0.07}$ & $\mathbf{74.23_{\pm0.19}}$ \\
ResNet-18
  & $80.52_{\pm0.37}$ & $\mathbf{81.10_{\pm0.08}}$
  & $78.91_{\pm0.27}$ & $\mathbf{81.03_{\pm0.12}}$
  & $80.82_{\pm0.14}$ & $\mathbf{81.12_{\pm0.14}}$
  & $80.70_{\pm0.21}$ & $\mathbf{81.29_{\pm0.24}}$ \\
WRN-28-10
  & $84.64_{\pm0.10}$ & $\mathbf{85.10_{\pm0.16}}$
  & $83.88_{\pm0.27}$ & $\mathbf{85.03_{\pm0.21}}$
  & $84.47_{\pm0.06}$ & $\mathbf{84.94_{\pm0.16}}$
  & $84.77_{\pm0.07}$ & $\mathbf{84.87_{\pm0.28}}$ \\
PyramidNet-110
  & $85.60_{\pm0.11}$ & $\mathbf{86.49_{\pm0.11}}$
  & $83.72_{\pm0.09}$ & $\mathbf{86.90_{\pm0.18}}$
  & $85.55_{\pm0.14}$ & $\mathbf{86.64_{\pm0.23}}$
  & $85.85_{\pm0.09}$ & $\mathbf{86.89_{\pm0.35}}$ \\
\midrule
\end{tabular}}
\caption{%
  Top-1 test accuracy (\%) on CIFAR-100 with Cutout augmentation.
}
\label{tab:cifar100_adaptive_vs_constant}
\vspace{-12pt}
\end{table}


\begin{table}[h]
\centering
\small
\setlength{\tabcolsep}{14pt}
\resizebox{\textwidth}{!}{%
\begin{tabular}{ll ccccc}
\midrule
\multirow{2}{*}{\textbf{Method}}
  & 
  & \multicolumn{5}{c}{\textbf{IMAGENET-}} \\
\cmidrule(lr){3-7}&
& 1K
& Real
& V2
& R
& Sketch \\

\toprule
SGD
     & --- &
        $76.72_{\pm0.09}$  & $83.02_{\pm0.02}$  & $64.40_{\pm0.07}$ & $19.95_{\pm0.12}$ & $20.56_{\pm0.20}$ \\
        
    \midrule

\multirow{2}{*}{SAM}
        & Const.
        & $77.45_{\pm 0.088}$ 
& $83.84_{\pm 0.08}$ 
& $65.34_{\pm 0.16}$ 
& $20.81_{\pm 0.34}$ 
& $22.12_{\pm 0.17}$ \\
        
     & Dynam.
       & $\mathbf{77.65_{\pm 0.11} }$
&$\mathbf{84.03_{\pm 0.07}}$
& $\mathbf{65.53_{\pm 0.18}}$
& $\mathbf{20.95_{\pm 0.07}}$ 
& $21.90_{\pm 0.18}$ \\
    \midrule

\multirow{2}{*}{GSAM}
    & Const. 
    & $77.37_{\pm 0.07}$ 
& $83.83_{\pm 0.07}$ 
& $65.42_{\pm 0.18}$ 
& $20.93_{\pm 0.25}$ 
& $21.26_{\pm 0.51}$ \\
       
    & Dynam.
       & $\mathbf{77.65_{\pm 0.15}}$
    &$\mathbf{83.98_{\pm 0.09}}$
    &$\mathbf{65.54_{\pm 0.22}}$
    & $\mathbf{21.14_{\pm 0.09}}$ 
    &$\mathbf{22.24_{\pm 0.16}}$  \\
 \midrule

\multirow{2}{*}{ASAM}
    & Const.
    &$77.11_{\pm 0.08}$ & $84.07_{\pm 0.12}$ & $65.06_{\pm 0.18}$ & $21.38_{\pm 0.18}$  & $21.67_{\pm 0.19}$\\
     
 & Dynam. 
    &$\mathbf{77.64_{\pm 0.17} }$ &$\mathbf{84.17_{\pm 0.19}} $ & $\mathbf{65.86_{\pm 0.09}}$ & $\mathbf{21.57_{\pm 0.41}}$  & $\mathbf{22.42_{\pm 0.68}}$ \\

\bottomrule
\end{tabular}}
\caption{
Top-1 accuracy (\%) on ImageNet and robustness datasets for ResNet-50.
}
\label{tab:imagenet_resnet50}
\vspace{-20pt}
\end{table}



\section{Conclusion}
In this paper, we investigated SAM and provided a new interpretation of the perturbation radius $\rho$ as a signal of the perturbation magnitude. Based on this finding, we proposed CAP-SAM,  which modifies the ascent 
step of SAM by adjusting the perturbation radius based on local sharpness estimates across steps. We theoretically investigated the impact of stochastic gradient noise on SAM's perturbation direction, showing that it becomes increasingly uninformative about local curvature near convergence, and demonstrated that CAP-SAM mitigates this undesirable effect. We validated CAP-SAM across several baselines as a plug-in for SAM and its variants, achieving consistent improvements in accuracy. 
\bibliography{references}

\appendix
\input{appendix}

\end{document}

%% file: rho_impact_misalignment.tex
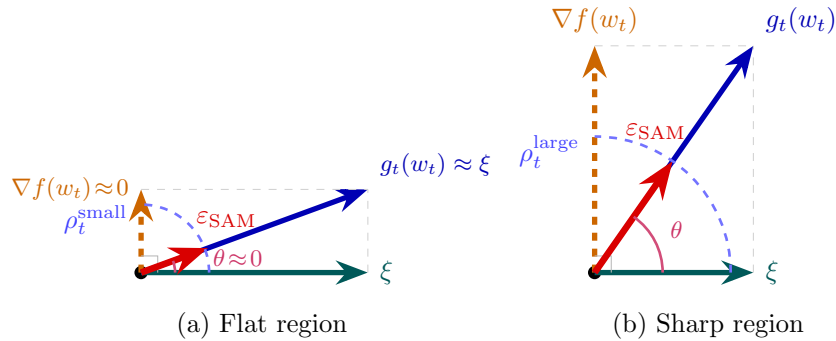
\begin{figure}[H]
\centering
\begin{tikzpicture}[
    scale=1.0,
    >=Stealth,
    every node/.style={font=\small},
    signal/.style={->, line width=2pt, color=teal!70!black},
    noise/.style={->, line width=2pt, color=orange!80!black, dashed},
    sumvec/.style={->, line width=2pt, color=blue!70!black},
    epssam/.style={->, line width=2.5pt, color=red!85!black},
]

\begin{scope}[xshift=-0cm, local bounding box=panelA]

  \pgfmathsetmacro{\Ra}{0.9}

  \coordinate (O) at (0,0);
  \fill (O) circle (2.5pt);

  \coordinate (Ga) at (3, 0);
  \coordinate (Xa) at (0, 1.1);
  \coordinate (Sa) at ($(Ga)+(Xa)$);

  \draw[signal] (O) -- (Ga) node[right] {$ \xi$};
  \draw[noise]  (O) -- (Xa)
      node[left, font=\footnotesize] {$\nabla f(w_t)\!\approx\!0$};
  \draw[sumvec] (O) -- (Sa)
      node[above right, font=\footnotesize]
      {$g_t(w_t)\approx\xi$};

  \draw[gray!35, thin, dashed] (Ga) -- (Sa);
  \draw[gray!35, thin, dashed] (Xa) -- (Sa);
  \draw[gray!55, thin] (0.22,0) -- (0.22,0.22) -- (0,0.22);

  \pgfmathsetmacro{\Sanorm}{veclen(3,0.55)}
  \coordinate (eSAMa) at ($(O)!{\Ra/\Sanorm}!(Sa)$);
  \draw[epssam] (O) -- (eSAMa)
      node[pos=0.8, above, font=\small, xshift=12pt, yshift=5pt]
      {$\varepsilon_{\mathrm{SAM}}$};

  \draw[blue!55, dashed, line width=1pt]
      (\Ra,0) arc[start angle=0, end angle=90, radius=\Ra];
  \node[blue!75, font=\small] at (-0.6, \Ra-0.2) {$\rho_t^{\mathrm{small}}$};

  \pgfmathsetmacro{\anga}{atan2(1.1,3)}
\draw[purple!70, line width=1pt]
    ({\Ra*0.5},0) arc[start angle=0, end angle=\anga, radius={\Ra*0.5}];
  \node[purple!80, font=\footnotesize] at
      ({1.3*cos(\anga/2)},{1.1*sin(\anga/2)}) {$\theta\!\approx\!0$};

  \node[font=\small, align=center] at (1.6, -0.7)
      {(a) Flat region};

\end{scope}

\begin{scope}[xshift=6.0cm, local bounding box=panelB]

   \pgfmathsetmacro{\Rb}{1.8}
 
  \coordinate (O) at (0,0);
  \fill (O) circle (2.5pt);
 
  \coordinate (G) at (2.1, 0);
  \coordinate (X) at (0, 3.0);
  \coordinate (S) at ($(G)+(X)$);   
 
  \draw[signal] (O) -- (G) node[right] {$ \xi$};
  \draw[noise]  (O) -- (X) node[above] {$\nabla f(w_t)$};
  \draw[sumvec] (O) -- (S) node[above right] {$g_t(w_t)$};
 
  \draw[gray!35, thin, dashed] (G) -- (S);
  \draw[gray!35, thin, dashed] (X) -- (S);
 
  \draw[gray!55, thin] (0.22,0) -- (0.22,0.22) -- (0,0.22);
 
  \pgfmathsetmacro{\Snorm}{veclen(2.1,3.0)}
  \coordinate (eSAM) at ($(O)!{\Rb/\Snorm}!(S)$);
  \draw[epssam] (O) -- (eSAM)
      node[pos=0.99, above,font=\small, xshift=-6pt, yshift=4pt] {$\varepsilon_{\mathrm{SAM}}$};
 
  \draw[blue!55, dashed, line width=1pt]
      (\Rb,0) arc[start angle=0, end angle=90, radius=\Rb];
  \node[blue!75, font=\small] at (-0.6, \Rb-0.2) {$\rho_t^{\mathrm{large}}$};
 
  \pgfmathsetmacro{\angS}{atan2(3.0,2.1)}
  \draw[purple!70, line width=1pt]
      ({0.5*\Rb}, 0)
      arc[start angle=0, end angle=\angS, radius=0.5*\Rb];
  \node[purple!80, font=\footnotesize] at
      ({0.68*\Rb*cos(\angS/2)},{0.68*\Rb*sin(\angS/2)}) {$\theta$};
 
  \node[font=\small, align=center] at (1.5, -0.7)
      {(b) Sharp region};
\end{scope}

\end{tikzpicture}
\caption{
Illustration of SAM perturbation directions in flat and sharp regions of the loss landscape.}
\label{fig:misalignment}
\label{fig:pyramidnet110-cifar100}

\end{figure}

%% file: convergence.tex
\pgfplotsset{compat=1.18}

\pgfplotsset{
  sam plot/.style={
    axis lines=left,
    tick align=outside,
    tick pos=left,
    grid=major,
    grid style={dotted, gray!40},
    label style={font=\small},
    tick label style={font=\small},
    xlabel style={font=\small},
    ylabel style={font=\small},
    legend pos=north east,
    legend style={draw=none, fill=none, font=\tiny},
    clip=false,
  },
  plotadaptive/.style={
    color={rgb,1:red,0.33;green,0.51;blue,0.67},
    line width=1.2pt,
    mark=none,
  },
  plotsam/.style={
    color={rgb,1:red,0.72;green,0.40;blue,0.36},
    line width=1.2pt,
    mark=none,
  },
  sgd/.style={
    color={rgb,1:red,0.42;green,0.57;blue,0.42},
    line width=1.2pt,
    mark=none,
  },
  cifarten/.style={
    color={rgb,1:red,0.76;green,0.56;blue,0.25},
    line width=1.2pt,
    mark=none,
  },
  cifarhundred/.style={
    color={rgb,1:red,0.55;green,0.42;blue,0.60},
    line width=1.2pt,
    mark=none,
  },
}

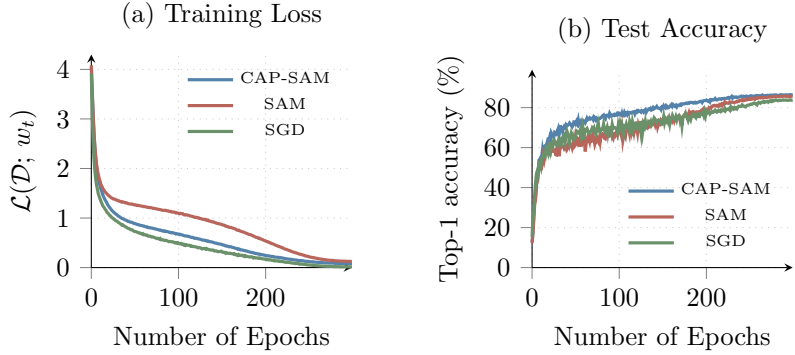
\begin{wrapfigure}{l}{0.71\textwidth}
     \vspace{-\intextsep}
     \centering
%
\begin{minipage}[t]{0.33\textwidth}
\centering
\begin{tikzpicture}[font=\footnotesize]
\begin{axis}[
  sam plot,
  width=\linewidth, height=4.4cm,
  xlabel={Number of Epochs},
  ylabel={$\mathcal{L}(\mathcal{D};\,w_t)$},
  title={\small (a) Training Loss},
  xmin=0, xmax=299,
  ymin=0, ymax=4.3
]
  \addplot[plotadaptive] table[x=epoch, y=train_loss, col sep=comma]
    {CIFAR/CIFAR100/CIFAR100_cutout-pyramidnet110-adaptive.csv};
  \addlegendentry{CAP-SAM}
  \addplot[plotsam] table[x=epoch, y=train_loss, col sep=comma]
    {CIFAR/CIFAR100/CIFAR100_cutout-pyramidnet110-sam.csv};
  \addlegendentry{SAM}
  \addplot[sgd] table[x=epoch, y=train_loss, col sep=comma]
    {CIFAR/CIFAR100/CIFAR100_cutout-pyramidnet110-sgd.csv};
  \addlegendentry{SGD}
\end{axis}
\end{tikzpicture}
\end{minipage}
\hspace{0.03\textwidth}
\begin{minipage}[t]{0.33\textwidth}
\centering
\begin{tikzpicture}[font=\footnotesize]
\begin{axis}[
  sam plot,
  width=\linewidth, height=4.2cm,
  xlabel={Number of Epochs},
  ylabel={Top-1 accuracy (\%)},
  title={\small (b) Test Accuracy},
  xmin=0, xmax=299,
  ymin=0, ymax=99,
  legend style={draw=none, fill=none, font=\tiny, legend pos=south east},
]
  \addplot[plotadaptive] table[x=epoch, y=test_acc1, col sep=comma]
    {CIFAR/CIFAR100/CIFAR100_cutout-pyramidnet110-adaptive.csv};
  \addlegendentry{CAP-SAM}
  \addplot[plotsam] table[x=epoch, y=test_acc1, col sep=comma]
    {CIFAR/CIFAR100/CIFAR100_cutout-pyramidnet110-sam.csv};
  \addlegendentry{SAM}
  \addplot[sgd] table[x=epoch, y=test_acc1, col sep=comma]
    {CIFAR/CIFAR100/CIFAR100_cutout-pyramidnet110-sgd.csv};
  \addlegendentry{SGD}
\end{axis}
\end{tikzpicture}
 \end{minipage}
\caption{%
  PyramidNet-110 on CIFAR-100.
  \textbf{(a)}~Training loss $\mathcal{L}(\mathcal{D};\,w_t)$ and
  \textbf{(b)}~test accuracy (top-1) for CAP-SAM, SAM, and SGD.
}
\label{fig:pyramidnet110-cifar100}

\end{wrapfigure}
 

%% file: hyperparameter.tex
\pgfplotsset{
  hyperparam/.style={
    axis lines=left,
    tick align=outside,
    tick pos=left,
    grid=none,
    clip=false,
    label style={font=\scriptsize},
    tick label style={font=\scriptsize},
    xlabel style={font=\scriptsize},
    ylabel style={font=\scriptsize},
    ylabel={Top-1 Acc. (\%)},
    ymin=96.0, ymax=97.0,
    ytick={96.0,96.3,96.6,96.9},
  },
  hypline/.style={
    color={rgb,1:red,0.26;green,0.53;blue,0.76},
    line width=1.0pt,
    dash pattern=on 4pt off 2pt on 1pt off 2pt,
    mark=*,
    mark size=3pt,
    mark options={
      fill={rgb,1:red,0.26;green,0.53;blue,0.76},
      draw={rgb,1:red,0.26;green,0.53;blue,0.76},
      solid,
    },
  },
}

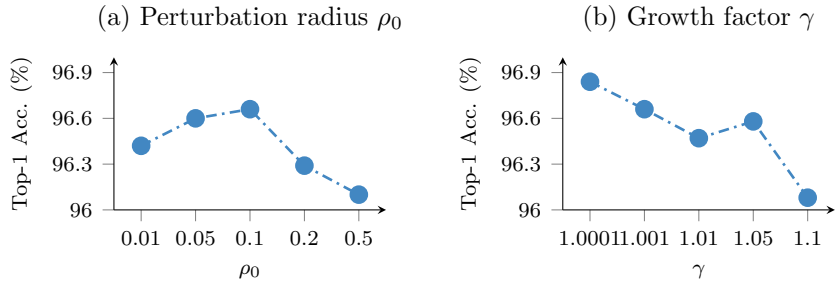
\begin{wrapfigure}{l}{0.75\textwidth}
 \vspace{-\intextsep}
 \setlength{\abovecaptionskip}{2pt}
 \setlength{\abovecaptionskip}{2pt}%
\captionsetup{singlelinecheck=false}
 \centering

\begin{minipage}[t]{0.34\textwidth}
\centering\vspace{0pt}
\begin{tikzpicture}[font=\footnotesize]
\begin{axis}[
  hyperparam,
  width=\linewidth, height=3.6cm,
  title={\small (a) Perturbation radius $\rho_0$},
  xlabel={$\rho_0$},
  xmin=0.5, xmax=5.5,
  xtick={1,2,3,4,5},
  xticklabels={0.01, 0.05, 0.1, 0.2, 0.5},
]
  \addplot[hypline] coordinates {
    (1, 96.42) (2, 96.60) (3, 96.66) (4, 96.29) (5, 96.10)
  };
\end{axis}
\end{tikzpicture}
\label{fig:hyper-rho}
\end{minipage}
\hspace{0.03\textwidth}
\begin{minipage}[t]{0.34\textwidth}
\centering\vspace{0pt}
\begin{tikzpicture}[font=\footnotesize]
\begin{axis}[
  hyperparam,
  width=\linewidth, height=3.6cm,
  title={\small (b) Growth factor $\gamma$},
  xlabel={$\gamma$},
  xmin=0.5, xmax=5.5,
  xtick={1,2,3,4,5},
  xticklabels={1.0001, 1.001, 1.01, 1.05, 1.1},
]
  \addplot[hypline] coordinates {
    (1, 96.84) (2, 96.66) (3, 96.47) (4, 96.58) (5, 96.08)
  };
\end{axis}
\end{tikzpicture}
\label{fig:hyper-gamma}
\end{minipage}
\par
\caption{%
  Effect of hyperparameters 
with ResNet-18 on CIFAR-10.
}
\label{fig:ablation-hyperparams}
    \vspace{-\intextsep}

\end{wrapfigure}

%% file: appendix.tex

\newpage

\section{Related Works}

Several works have investigated the correlation between the generalization ability of a model and the sharpness of the minimum it converges to \cite{DBLP:journals/corr/DinhPBB17,DBLP:journals/corr/KeskarMNST16,wang2023sharpnessawaregradientmatchingdomain}. 

Following this line of work, subsequent studies have examined the connection between batch size and generalization, showing that the choice of batch size influences the tendency of a model to converge to a sharp or flat minimum \cite{DBLP:journals/corr/KeskarMNST16}.
Different measures have also been proposed to evaluate the geometric characteristics of the loss landscape and their connection to sharpness \cite{Measures,kim2023fantastic,zhuang2022surrogate}. 
These works justify the goal of seeking flatter valleys in the loss landscape to enhance generalization. Inspired by this, Sharpness-Aware Minimization (SAM) \cite{sam} was proposed as a strategy that targets flatter minima by minimizing the worst-case loss within a local neighborhood. SAM has been shown to achieve generalization improvements across a variety of tasks and benchmarks \cite{caldarola2025beyond,kwon2021asam,DBLP:journals/corr/abs-2010-04925}.    


A series of SAM variants have since been proposed, addressing SAM's limitations from different perspectives. From the perspective of the neighborhood's geometric measure, ASAM \cite{kwon2021asam} introduce scale invariance to address SAM's sensitivity to weight re-scaling, while Fisher SAM \cite{kim22f} replace SAM's Euclidean balls with ellipsoids induced by the Fisher information metric. From the perspective of the optimization stability, GSAM \cite{zhuang2022surrogate} jointly minimizes the perturbed loss and a surrogate gap, while the friendly-adversary approach of VaSSO \cite{li2023enhancing} stabilizes the adversarial perturbation through variance suppression. To reduce computational cost, a further line of work seeks to alleviate SAM's characteristic doubling of forward–backward passes. For instance, ESAM \cite{esam} reduces this overhead through stochastic weight perturbation and sharpness-sensitive data selection, and Sparse SAM \cite{ssam} perturbs only a sparse subset of parameters via a learned binary mask.


Despite their improvements over the vanilla formulation, these variants all operate within the original min-max framework, refining either the perturbation direction or the regularization term while treating the radius as a fixed hyperparameter shared across all iterations and regions of the landscape. The strategy proposed in this work, CAP-SAM, instead adapts the size of the perturbation radius itself based on local sharpness estimates, bridging the geometric and stability perspectives in a direction largely unexplored by existing SAM variants.

\section{Gradient misalignment}  
\label{gradientmismisalignment}
We analyze the behavior of the SAM perturbation near convergence by decomposing the stochastic gradient on a mini-batch $B$ as:
\[ g(w_t) =  \nabla f(w_t) +  \xi,\]
where $ \xi$ is a zero-mean Gaussian random variable representing the batch-specific noise . Assuming $g(w_t)$ is an unbiased estimator of the full gradient, we have $\mathbb{E}_B[ \xi] = 0$. Moreover, as $ \nabla f(w_t)$ is deterministic with respect to batch sampling, 
\[
\mathbb{E}_B\langle  \nabla f(w_t),\,  \xi \rangle
=
\langle  \nabla f(w_t),\, \mathbb{E}[\xi] \rangle = 0.
\]

Consequently, the stochastic noise has zero mean and zero expected projection onto the true gradient: on average it carries no systematic descent or curvature information. Figure \ref{fig:vectormisalignment}(a) illustrates this expected decomposition, where the stochastic gradient is expressed as the sum of the deterministic gradient and a zero-mean noise component.

The SAM perturbation at step $t$ is
\[
\epsilon_t
=
\rho\,
\frac{g(w_t)}{\|g(w_t)\|}
=
\rho\,
\frac{ \nabla f(w_t) + \xi}
     {\| \nabla f(w_t) + \xi\|}.
\]   
As optimization approaches minimum, where $\| \nabla f(w_t)\| \to 0$, the deterministic gradient is dominated by the noise 
and the perturbation collapses to
\[
\epsilon_t
\approx
\rho\,
\frac{\xi}{\|\xi\|},
\]
The perturbation direction becomes increasingly governed by stochastic gradient noise rather than by the deterministic descent direction, making the perturbation uninformative about the loss curvature .

\input{vector_misalignment}

\subsection{Geometric characterization of the misalignment}
To quantify this behavior, we measure the alignment between $\epsilon_t$ and the noise direction $\epsilon_{noise}=\rho\,\frac{\xi}{\|\xi\|}$.
The  alignment angle $\theta$ is given as the inverse cosine of
\[\cos\theta 
    \;=\; 
    \left\langle 
        \epsilon_t,\;
        \epsilon_{noise}
    \right\rangle
    \;=\;
    \left\langle 
        \frac{g(w_t)}
             {\|g(w_t)\|},\;
        \frac{\xi}{\|\xi\|}
    \right\rangle,
   \]
\noindent where $\cos\theta = 1$ indicates $\theta = 0$, i.e., perfect alignment.

Near a minimum, $\| \nabla f(w_t)\| \approx 0$,  so the gradient component vanishes and the expression simplifies to
\[\cos\theta 
    \;\approx\; 
    \left\langle 
        \frac{\xi}{\|\xi\|},\;
        \frac{\xi}{\|\xi\|}
    \right\rangle
    \approx
    \frac{\|\xi\|^2}{\|\xi\|^2}
    = 1
    \qquad \Longrightarrow \qquad \theta = 0.
   \]

Therefore, the stochastic gradient becomes increasingly aligned with the noise component, causing the SAM perturbation to converge to the normalized noise direction. Consequently, the ascent step is driven primarily by stochastic fluctuations rather than by the underlying loss geometry, reducing its effectiveness as an estimator of local sharpness, as illustrated in Figure \ref{fig:vectormisalignment} (b).

\section{Algorithm}

\begin{algorithm}[H]
\caption{CAP-SAM}
\label{alg:algorithm}
\begin{algorithmic}[1]
\STATE \textbf{Input:} Training set $D = \{(x_i,y_i)\}_{i=1}^n$,
      loss $L$, learning rate $\lambda$, growth factor $\gamma$,
      bounds $\rho_{\min} \leq \rho_{\max}$, initial radius $\rho_0$.

\STATE Initialise weights $w \leftarrow w_0$
   
\STATE Initialise perturbation radius $\rho \leftarrow \rho_0$

\WHILE {not converged}
    \STATE Sample a mini-batch $B$ data form $D$
    \STATE Compute stochastic gradient $g_t$
     \STATE Compute $r_t = \frac {\|g(w_t)\|}{\|g(w_{t-1})\|}$ 
    \STATE Compute 
$
\rho_t =
\begin{cases}
        \min\!\left(\gamma\,\rho_{t-1},\;\rho_{\max}\right),        & \mbox{if }r_t \geq 1 \\[4pt]
        \max\!\left(\dfrac{\rho_{t-1}}{\gamma},\;\rho_{\min}\right), & \text{otherwise}
      \end{cases}
$
    \STATE Compute $\epsilon_t = \rho_t {g(x_k) }~/~{\|g(x_k) \|_2}$
    \STATE Update weights $w_{t+1} \leftarrow w_t - \lambda\, g(w_t+\epsilon_t)$
\ENDWHILE
\STATE \textbf{Return} $w,\ g(w_{t-1})$
\end{algorithmic}
\end{algorithm}

\section{Theoretical Analysis}

\subsection{$r_t$ as a measure of local sharpness}\label{r_approximation}
We provide a formal justification for $r_t$ as a principled measure of how local sharpness evolves along the optimization trajectory. 
Following \cite{zhuang2022surrogate}, we define the surrogate gap at a point $w$ as:
 \[
h(w)
=
\max_{\|\delta\| \le \rho}
f(w+\delta)-f(w)
\approx
f^{SAM}(w)-f(w).
\]

\noindent The surrogate gap measures the difference between the maximum loss within the neighborhood and the loss at the center point. 

\begin{theorem}
Let $w^*$ be a local minimum of $f$, and let $\sigma^{(t)}_{\max}\ $ denote the dominant eigenvalue of the Hessian as a measure of sharpness. Assuming the iterates $w_{t-1}$ and $w_t$ lie near $w^\ast$, the ratio $r_t$ defined in Eq.\eqref{rho}  approximates the 
change in sharpness between consecutive steps: $\sigma^{(t)}_{\max} ~/~ \sigma^{(t-1)}_{\max}\approx r_t ~/~ \alpha^2$ where $\alpha =\rho_t ~/~ \rho_{t-1} $ is the ratio of perturbation radii between iterations.
\end{theorem}

\begin{proof}
We follow the argument of \cite{zhuang2022surrogate}, adapting it to the ratio of
surrogate gaps at consecutive iterates.
Since $\rho_t$ is small, consider the second-order Taylor expansion of $f$ around 
$w_t$: 
\[
f^{SAM}(w_t)=f(w_t+\delta)=f(w_t)+
\nabla f(w_t)^\top \delta+\frac{1}{2}\delta^\top H_t \delta+\mathcal{O}(\|\rho\|^3),
\] 
where $ H_t = \nabla^2 f(w_t)$ is the Hessian at $w_t$, assumed to be positive semidefinite near local minima. Since $w_t$ lies near the local minimum, the gradient can be assumed to be negligible, $\nabla f(w_t)\approx 0$, so the first-order term vanishes.
Normalizing $\delta$ such that $\|\delta\|=\rho_t$, we obtain  
\[
f^{SAM}(w_t)-f(w_t)
=
\max_{\|\delta\| \le \rho_t}
f(w_t+\delta)-f(w_t)
=
\frac{1}{2}\sigma_{\max}^{(t)}\rho_t^2
+
\mathcal{O}(\rho_t^3).
\]

\noindent  where $\lambda_{\max}^{(t)}$ is the largest eigenvalue of the hessian $H_t$. 
Applying the same expansion at step $t-1$ and substituting into $r_t= h(w_t)/h(w_{t-1})$: 
\[
r_t
=\frac{f^{SAM}(w_t)-f(w_t)}{f^{SAM}(w_{t-1})-f(w_{t-1})}=
\frac{
\frac{1}{2}\lambda_{\max}^{(t)}\rho_t^2
+
\mathcal{O}(\rho_t^3)
}{
\frac{1}{2}\lambda_{\max}^{(t-1)}\rho_{t-1}^2
+
\mathcal{O}(\rho_{t-1}^3)
}
\]
Defining $\alpha=
\frac{\rho_t}{\rho_{t-1}}$ and rearranging gives
$
\frac{\lambda_{\max}^{(t)}}
     {\lambda_{\max}^{(t-1)}}
\approx
r_t ~/~\alpha^2
$
\end{proof}

\begin{remark}[First-order interpretation and computational cost] 
The Theorem above establishes $r_t$ as a sharpness measure via a second-order argument. Under the adaptive update rule, however, $r_t$ also admits a simpler interpretation that makes it cheap to compute in practice. Using a first-order Taylor expansion
\[
|f(w_t + \epsilon_t) - f(w_t)|
\approx
|\langle g(w_t), \epsilon_t \rangle|
=
\rho_t \|g(w_t)\|,
\]
where $g(w_t) = \nabla f(w_t)$ and the SAM perturbation is 
$\epsilon_t = \rho_t g(w_t)/\|g(w_t)\|$. Therefore:
\[
r_t
\approx
\frac{\rho_t \|g(w_t)\|}{\rho_{t-1}\|g(w_{t-1})\|}
=
\frac{\rho_t}{\rho_{t-1}}
\cdot
\frac{\|g(w_t)\|}{\|g(w_{t-1})\|}.
\]
Since the adaptive update rule adjusts $\rho$ by a factor of 
$\gamma = 1.001$, consecutive radii satisfy 
$\rho_t/\rho_{t-1} \in \{\gamma, \gamma^{-1}\} = 1 + \mathcal{O}(10^{-3})$, 
and thus:
\[
r_t \approx \frac{\|g(w_t)\|}{\|g(w_{t-1})\|}.
\]
Hence $r_t$ is well approximated by the ratio of consecutive gradient norms \cite{ratio}, a quantity already available during training at no additional computational cost.
\end{remark}

\subsection{Assumptions}
\label{app:assumptions}
Our analysis relies on the following assumptions, which are standard in
nonconvex stochastic optimization.

\begin{assumption}[Bounded Gradient]
$\exists~G \geq 0$ such that $\|\nabla f(w)\| \leq G$.
\label{ass:bounded_gradient}
\end{assumption}

\begin{assumption}[Bounded Variance]
$\exists ~\sigma \geq 0$ such that
$\mathbb{E}\left[\|g(w) - \nabla f(w)\|^2\right] \leq \sigma^2$.
\label{ass:bounded_variance}
\end{assumption}

\begin{assumption}[L-smoothness]
$\exists~L > 0$ such that
$\|\nabla f(w) - \nabla f(v)\| \leq L\|w - v\|,~\forall w, v \in \mathbb{R}^d$.
\label{ass:smoothness}
\end{assumption}

\paragraph{Decaying-radius assumption.}
\label{decaying_radius}
Following existing convergence analyses of SAM and its variants \cite{zhuang2022surrogate,ssam}, we assume that the perturbation radius vanishes asymptotically, i.e., $\rho_t=\rho_{\max}/\sqrt{t}$. As in prior work, the assumption is made for analytical tractability and does not reflect how the method is used in practice. In our experiments (Section~\ref{experiments}), $\rho_t$ follows Eq.~\eqref{rho} within the fixed interval $[\rho_{\min}, \rho_{\max}]$ and is not forced to decay.

\subsection{Proof of Theorem \ref{thm:convergence}}
\label{theorem}
We adapt the standard SAM convergence proof \cite{sam} to the  CAP-SAM formulation. 
We can write the algorithm \ref{alg:algorithm} :
\begin{equation}
    \left\{
    \begin{aligned}
        w_{t+\frac{1}{2}} &= w_t + \epsilon_t \\[4pt]
        w_{t+1}           &= w_t - \eta \cdot g\!\left(w_{t+\frac{1}{2}}\right)
    \end{aligned}
    \right.
    \label{eq:sam_iteration}
\end{equation}
where for SAM $\epsilon_t =\rho \cdot \frac{g(w_t)}{\|g(w_t)\|} $   and for CAP-SAM $\epsilon_t =\rho_t \cdot \frac{g(w_t)}{\|g(w_t)\|} $  with $ \rho_t$ as defined in Eq. \eqref{rho}
              
We first present useful lemmas that will be used to support 
the proof of the main result.
\begin{lemma}
\label{Lemma1}
For any $\rho > 0$, $L > 0$, and differentiable function $f$, we have:
\begin{equation}
    \left\langle
        \nabla f(w_t),\,
        \nabla f\!\left(w_t + \rho\,
            \frac{\nabla f(w_t)}{\|\nabla f(w_t)\|}
        \right)
    \right\rangle
    \;\geq\;
    \|\nabla f(w_t)\|^2 - \rho L G.
\end{equation}
\end{lemma}
\begin{proof}
Adding and subtracting $\|\nabla f(w_t)\|$ inside the inner product 
and factoring out $\|\nabla f(w_t)\|/\rho$ :
\begin{align}
\mathrm{LHS}
    &= \frac{\|\nabla f(w_t)\|}{\rho}
       \left\langle
       \frac{\rho\,\nabla f(w_t)}{\|\nabla f(w_t)\|},\;
       \nabla f\!\left(w_t 
       + \rho\frac{\nabla f(w_t)}{\|\nabla f(w_t)\|}\right)
       - \nabla f(w_t)
       \right\rangle
       + \|\nabla f(w_t)\|^2.
\end{align}
\begin{align}
\mathrm{LHS}
    &\geq
    - L\,\frac{\|\nabla f(w_t)\|}{\rho}
      \left\|\frac{\rho}{\|\nabla f(w_t)\|}
      \nabla f(w_t)\right\|^2
    + \|\nabla f(w_t)\|^2.
    \notag\\[4pt]
    &\geq
    - L\,\frac{\|\nabla f(w_t)\|}{\rho} \cdot \rho^2
    + \|\nabla f(w_t)\|^2
    \notag\\[4pt]
    &= - \|\nabla f(w_t)\|\,\rho L + \|\nabla f(w_t)\|^2.
    \notag
\end{align}
Applying the bounded gradient Assumption~\ref{ass:bounded_gradient}
\begin{align}
    \|\nabla f(w_t)\|^2 - \rho L\,\|\nabla f(w_t)\|
    \;\geq\;
    \|\nabla f(w_t)\|^2 - \rho L G,
\end{align}

\end{proof}

\begin{lemma}
\label{Lemma2}
For $\rho > 0$, $L > 0$, the iteration~\ref{eq:sam_iteration} satisfies the following 
inequality:
\[
    \mathbb{E}\left\langle \nabla f(w_t),\, 
    g\!\left(w_{t+\frac{1}{2}}\right) 
    \right\rangle
    \;\geq\;
    \frac{1}{2}\|\nabla f(w_t)\|^2 - L^2\rho^2 - L\rho G.
\]
\end{lemma}

\begin{proof}
We denote the deterministic value of $w_{t+\frac{1}{2}}$ as defined in ~\ref{eq:sam_iteration}. Adding and subtracting 
$g\!\left(\hat{w}_{t+\frac{1}{2}}\right)$  decomposes the inner product into two terms: 
\begin{align}
    \left\langle \nabla f(w_t),\, 
    g\!\left(w_{t+\frac{1}{2}}\right) \right\rangle
    &= \underbrace{\left\langle \nabla f(w_t),\;
       g\!\left(w_t + \rho\frac{g(w_t)}{\|g(w_t)\|}\right)
       - g\!\left(\hat{w}_{t+\frac{1}{2}}\right)
       \right\rangle}_{\text{Term I: perturbation direction error}}
       \notag\\[4pt]
    &\quad + \underbrace{\left\langle 
       g\!\left(\hat{w}_{t+\frac{1}{2}}\right),\,
       \nabla f(w_t) \right\rangle}_{\text{Term II: 
       true gradient alignment}}.
       \notag
\end{align}
Bounding Term~I from below by applying Young's inequality, followed by $L$-smoothness: 
\begin{align}
    -\left\langle \nabla f(w_t),\;
      g\!\left(w_t + \rho\frac{g(w_t)}{\|g(w_t)\|}\right)
      - g\!\left(\hat{w}_{t+\frac{1}{2}}\right)
      \right\rangle
    &\leq \frac{1}{2}
      \left\|g\!\left(w_t + \rho\frac{g(w_t)}{\|g(w_t)\|}\right)
      - g\!\left(\hat{w}_{t+\frac{1}{2}}\right)\right\|^2
      + \frac{1}{2}\|\nabla f(w_t)\|^2.
      \notag
\end{align}
\begin{align}
    &\leq \frac{L^2}{2}
      \left\|w_t + \rho\frac{g(w_t)}{\|g(w_t)\|}
      - \hat{w}_{t+\frac{1}{2}}\right\|^2
      + \frac{1}{2}\|\nabla f(w_t)\|^2.
      \notag
\end{align}
\begin{align}
    &= \frac{L^2}{2}
      \left\|w_t + \rho\frac{g(w_t)}{\|g(w_t)\|}
      - \left(w_t + \rho\frac{\nabla f(w_t)}{\|\nabla f(w_t)\|}\right)
      \right\|^2
      + \frac{1}{2}\|\nabla f(w_t)\|^2 
      \notag \\[6pt]
    &= \frac{L^2\rho^2}{2}
      \left\|\frac{g(w_t)}{\|g(w_t)\|}
      - \frac{\nabla f(w_t)}{\|\nabla f(w_t)\|}\right\|^2
      + \frac{1}{2}\|\nabla f(w_t)\|^2.
      \notag
\end{align}
\begin{align}
    \;\geq\; 
    -L^2\rho^2 - \frac{1}{2}\|\nabla f(w_t)\|^2.
    \notag
\end{align}

For the second term applying Lemma~\ref{Lemma1} directly, we have 
\begin{align}
    \text{Term II} \;\geq\; 
    \|\nabla f(w_t)\|^2 - \rho L G.
    \notag
\end{align} 

Combining the two terms: 
\begin{align}
    \mathbb{E}\left\langle \nabla f(w_t),\, 
    g\!\left(w_{t+\frac{1}{2}}\right) \right\rangle
    &= \text{Term I} + \text{Term II}
    \notag\\[4pt]
    &\geq \left(-L^2\rho^2 
      - \frac{1}{2}\|\nabla f(w_t)\|^2\right)
      + \left(\|\nabla f(w_t)\|^2 - \rho L G\right)
    \notag\\[4pt]
    &= \frac{1}{2}\|\nabla f(w_t)\|^2 
      - L^2\rho^2 - L\rho G,
    \notag
\end{align}
\end{proof}

\begin{lemma}
\label{Lemma3}
For $\eta \leq \frac{1}{L}$, the iteration~\ref{eq:sam_iteration} satisfies for all $t > 0$:
\[
    \mathbb{E}f(w_{t+1}) \;\leq\;
    \mathbb{E}f(w_t)
    - \frac{\eta}{2}\mathbb{E}\|\nabla f(w_t)\|^2
    + L\eta^2\sigma^2
    + \eta L^2\rho^2
    + (1 - L\eta)\eta LG\rho.
\]
\end{lemma}

\begin{proof}
By the $L$-smoothness of $f$ and the update rule \ref{eq:sam_iteration}:
\begin{align}
f(w_{t+1})
    &\leq f(w_t)
    - \eta\left\langle \nabla f(w_t),\,
      g\!\left(w_{t+\frac{1}{2}}\right)\right\rangle
    + \frac{L\eta^2}{2}
      \left\|g\!\left(w_{t+\frac{1}{2}}\right)\right\|^2.
      \label{eq:smooth_descent}
\end{align}
Substituting the identity
$\left\|g\!\left(w_{t+\frac{1}{2}}\right)\right\|^2
= \left\|\nabla f(w_t) 
  - g\!\left(w_{t+\frac{1}{2}}\right)\right\|^2
- \|\nabla f(w_t)\|^2
+ 2\left\langle \nabla f(w_t),\,
  g\!\left(w_{t+\frac{1}{2}}\right)\right\rangle$
into~\ref{eq:smooth_descent} and rearranging:
\begin{align}
f(w_{t+1})
    &\leq f(w_t)
    - \frac{L\eta^2}{2}\|\nabla f(w_t)\|^2
    + \frac{L\eta^2}{2}
      \left\|\nabla f(w_t) 
      - g\!\left(w_{t+\frac{1}{2}}\right)\right\|^2
    - (1 - L\eta)\eta
      \left\langle \nabla f(w_t),\,
      g\!\left(w_{t+\frac{1}{2}}\right)\right\rangle
      \notag\\
       &\leq f(w_t)
    - \frac{L\eta^2}{2}\|\nabla f(w_t)\|^2
    + L\eta^2\|\nabla f(w_t) - g(w_t)\|^2
    + L\eta^2
      \left\|g(w_t) 
      - g\!\left(w_{t+\frac{1}{2}}\right)\right\|^2
    \notag\\
    &\quad
    - (1-L\eta)\eta
      \left\langle \nabla f(w_t),\,
      g\!\left(w_{t+\frac{1}{2}}\right)\right\rangle.
       \notag
\end{align}
Taking expectation and applying Lemma \ref{Lemma2}: 
\begin{align}
\mathbb{E}f(w_{t+1})
    &\leq \mathbb{E}f(w_t)
    - \frac{L\eta^2}{2}\mathbb{E}\|\nabla f(w_t)\|^2
    + L\eta^2\sigma^2
    + \eta^2 L^3\rho^2
    - (1-L\eta)\eta\,
      \mathbb{E}\!\left\langle \nabla f(w_t),\,
      g\!\left(w_{t+\frac{1}{2}}\right)\right\rangle.
      \notag \\
    &\leq \mathbb{E}f(w_t)
    - \frac{L\eta^2}{2}\mathbb{E}\|\nabla f(w_t)\|^2
    + L\eta^2\sigma^2
    + \eta^2 L^3\rho^2
    - (1-L\eta)\eta
      \left[\frac{1}{2}\mathbb{E}\|\nabla f(w_t)\|^2
      - L^2\rho^2 - L\rho G\right].
    \notag
\end{align}
Collecting terms in $\mathbb{E}\|\nabla f(w_t)\|^2$ and $\rho^2$ yields:
\begin{align}
\mathbb{E}f(w_{t+1})
    &\leq \mathbb{E}f(w_t)
    - \frac{\eta}{2}\mathbb{E}\|\nabla f(w_t)\|^2
    + L\eta^2\sigma^2
    + \eta L^2\rho^2
    + (1-L\eta)\eta LG\rho.
    \notag
\end{align}
which completes the proof.
\end{proof}

\begin{proposition}
\label{Prop1}
Let $\eta_t = \eta_0/\sqrt{t}$ and perturbation amplitude $\rho$ decay 
with the square root of $t$, e.g., $\rho_t = \rho_{\max}/\sqrt{t}$. For $\rho_{\max} \leq G\eta_0$ and $\eta_0 \leq \frac{1}{L}$, we have:
\[
    \frac{1}{T}\sum_{t=1}^{T}
    \mathbb{E}\|\nabla f(w_t)\|^2
    \;\leq\;
    C_1\,\frac{1}{\sqrt{T}}
    + C_2\,\frac{\log T}{\sqrt{T}},
\]
where
\[
    C_1 \;=\; \frac{2}{\eta_0}
    \Bigl(f(w_0) - \mathbb{E}f(w_T)\Bigr),
    \qquad
    C_2 \;=\; 2\!\left(L\sigma^2\eta_0 
    + LG\rho_{\max}\right).
\]
\end{proposition}

\begin{proof}
Under the decaying-radius assumption \ref{decaying_radius}, $\rho_t = \rho_{\max}/\sqrt{t}$. Since all three lemmas hold for any $\rho > 0$, they apply with $\rho$ replaced by $\rho_t$ at each step. As $\rho_t \leq \rho_{\max}/\sqrt{t}$ for all $t$, we set $\rho_0 = \rho_{\max}$, corresponding to the worst case in which $\rho_t$ always attains its upper bound. Replacing $\eta_t = \eta_0/\sqrt{t}$ and $\rho_t = \rho_{\max}/\sqrt{t}$ into Lemma~\ref{Lemma3}:
\begin{align}
    \mathbb{E}f(w_{t+1})
    &\leq \mathbb{E}f(w_t)
    - \frac{\eta_t}{2}\mathbb{E}\|\nabla f(w_t)\|^2
    + L\eta_t^2\sigma^2
    + \eta_t L^2\rho_t^2
    + (1 - L\eta_t)\eta_t LG\rho_t.
    \notag
\end{align}
Summing the per-step inequality from $t=1$ to $T$ 
and rearranging, the left-hand side telescopes to 
$\mathbb{E}f(w_T) - f(w_0)$:
\begin{align}
    \mathbb{E}f(w_T) - f(w_0)
    &\leq
    - \sum_{t=1}^{T}\frac{\eta_t}{2}
      \mathbb{E}\|\nabla f(w_t)\|^2
    + \left(L\sigma^2\eta_0^2 + LG\rho_{\max}\eta_0\right)
      \sum_{t=1}^{T}\frac{1}{t}
    \notag\\[4pt]
    &\quad
    + \left(L^2\eta_0\rho_{\max}^2 - L^2G\eta_0^2\rho_{\max}\right)
      \sum_{t=1}^{T}\frac{1}{t^{3/2}}.
    \notag
\end{align}
Drop the last term using $\rho_{\max} \leq G\eta_0$
\begin{align}
    \mathbb{E}f(w_T) - f(w_0)
    &\leq
    - \sum_{t=1}^{T}\frac{\eta_t}{2}
      \mathbb{E}\|\nabla f(w_t)\|^2
    + \left(L\sigma^2\eta_0^2 + LG\rho_{\max}\eta_0\right)
      \sum_{t=1}^{T}\frac{1}{t}.
    \notag
\end{align}
With 
\begin{align}
    \frac{\eta_T}{2}\sum_{t=1}^{T}
    \mathbb{E}\|\nabla f(w_t)\|^2
    &\leq
    \sum_{t=1}^{T}\frac{\eta_t}{2}
    \mathbb{E}\|\nabla f(w_t)\|^2
    \notag
\end{align}
where replacing $\eta_T = \frac{\eta_0}{\sqrt{T}}$ yields: 
\begin{align}
    \frac{\eta_0}{2\sqrt{T}}\sum_{t=1}^{T}
    \mathbb{E}\|\nabla f(w_t)\|^2
    &\leq
    f(w_0) - \mathbb{E}f(w_T)
    + \left(L\sigma^2\eta_0^2 + LG\rho_{\max}\eta_0\right)
      \sum_{t=1}^{T}\frac{1}{t}
      \\
     &\leq
    f(w_0) - \mathbb{E}f(w_T)
    + \left(L\sigma^2\eta_0^2 + LG\rho_{\max}\eta_0\right)\log T.
    \notag
\end{align}
Finally, we archive the result: 
\begin{align}
    \frac{1}{T}\sum_{t=1}^{T}
    \mathbb{E}\|\nabla f(w_t)\|^2
    &\leq
    \frac{2(f(w_0) - \mathbb{E}f(w_T))}{\eta_0}
    \frac{1}{\sqrt{T}}
    + 2\left(L\sigma^2\eta_0 + LG\rho_{\max}\right)
    \frac{\log T}{\sqrt{T}}
    \notag\\[6pt]
    &= C_1\,\frac{1}{\sqrt{T}} 
    + C_2\,\frac{\log T}{\sqrt{T}},
    \notag
\end{align}

This shows that,under the decaying-radius assumption standard in SAM analyses, CAP-SAM  converges at the rate of 
$\mathcal{O}(\log T / \sqrt{T})$.
\end{proof}

\section{Experimental setup}
\label{Exp_setup}
We report the averaged results and standard deviation results over 3 independent runs.

\subsection{CIFAR} 
For CIFAR-10 and CIFAR-100 \cite{krizhevsky2009learning}, we compare test accuracy against SAM and three variants ASAM \cite{kwon2021asam}, Fisher SAM\cite{kim22f}, and VASSO \cite{li2023enhancing} across four neural network architectures VGG-11 \cite{vgg}, ResNet-18 \cite{he2016deep}, WRN-28-10 \cite{DBLP:journals/corr/ZagoruykoK16} and PyramidNet-110 \cite{DBLP:journals/corr/HanKK16}. Standard data augmentation is applied, including random crop, random horizontal flip, normalization, and Cutout \cite{DBLP:journals/corr/abs-1708-04552}. The first three models are trained for 200 epochs with a batchsize of 128, and PyramidNet-110 is trained for 300 epochs. A cosine learning rate schedule is used across all settings, with an initial learning rate of 0.05 and weight decay of 0.001 .  For CIFAR-10, we set $\rho =0.1$ for SAM, Fisher SAM, and VASSO, and $\rho =0.05$ for ASAM, following the original papers. For our adaptive formulation, we use a initial radius $\rho_0 =0.05$ across all architectures, with the exception of ResNet-18 where $\rho_0 =0.1$ , and growth factor $\gamma = 1.001$. 
For CIFAR-100,we set $\rho =0.2$ for SAM, Fisher SAM, and VASSO, and $\rho =1$ for ASAM. For our adaptive formulation, we use a initial radius $\rho_0 =0.2$ across all architectures, and growth factor $\gamma = 1.001$.
\subsection{ImageNet}
For ImageNet-1K , we train ResNet-50 from scratch and evaluate on the ImageNet validation set as well as four additional out-of-distribution test sets: ImageNet-Real \cite{beyer2020we}, ImageNet-V2\cite{recht2019imagenet}, ImageNet-R \cite{hendrycks2021many}, and ImageNet-Sketch \cite{wang2019learning}. Following the setup of GSAM \cite{zhuang2022surrogate} and \cite{chen2021vision}, without strong data augmentation. Models are trained with a global batch size of 4096, a perturbation batch size of $m = 64$. A linear learning rate schedule with warmup is used, with an initial
learning rate of 0.016 and weight decay of 0.001. We evaluate the proposed CAP-SAM formulation applied to SAM, ASAM \cite{kwon2021asam}, and GSAM \cite{zhuang2022surrogate}. Following hyperparameter tuning, we set $\rho =0.04$ for SAM, $\rho =0.05$ for GSAM and $\rho =0.5$ for ASAM. The same values are used as the initial radius for the corresponding adaptive formulations, with growth factor  $\gamma = 1.001$. 

\subsection{Statistical analysis}
Statistical analysis for each test case (CIFAR-10, CIFAR-100 and ImageNet) was performed using an analysis of variance (ANOVA) model, with the perturbation type (constant $\rho$, dynamic $\rho$) as the experimental factor of interest. In the case of CIFAR-10 and CIFAR-100, the Method (SAM, ASAM, FSAM, VaSSO) and the Model (VGG-11-BN, ResNet-18, WRN-28-10, PyramidNet-110) were used as blocking variables, so that their marginal effects on the accuracy are accounted for when estimating the effect of deploying the dynamic $\rho$ perturbation. For ImageNet, two comparisons were performed: the first used the Method (SAM, GSAM and ASAM) and the ImageNet Version (1K, Real, V2, R, Sketch) as blocking variables, to directly estimate the effects of the dynamic $\rho$. The second compared all combinations of method + perturbation strategy directly against the SGD baseline, with Version as a blocking variable.

All statistical analyses were performed in R version 4.6.0. The p-values for the comparisons between constant and dynamic $\rho$ were estimated using function \texttt{aov}. Post-ANOVA tests used Tukey contrasts as implemented in function \texttt{glht} from package \textit{multcomp}. Effect size and confidence interval estimation were obtained applying function \texttt{confint} to the Tukey comparison models. The ANOVA tables for the three comparisons are presented below.

\begin{table} [H]
\centering
\setlength{\tabcolsep}{5pt}
\begin{tabular}{lccccc}
\toprule
\multirow{2}{*}{Factor}  & \multirow{2}{*}{degrees-of-freedom} & \multirow{2}{*}{F-value}  & \multirow{2}{*}{p-value}   & Mean acc. gain ($\%$)& \multirow{2}{*}{$CI_{95}$}\\
&&&& (dynam. - const.)&\\
\midrule
Variant & $1$                  &$19.51$     &$\mathbf{1.83\times10^{-4}}$&0.23&$[0.12,0.33]$\\
Method  & $3$                  &$0.451$ &0.72&---&---\\
Model   & $3$                  &$1107.6$ &$\mathbf{7.42\times10^{-26}}$&---&---\\
Residuals  & $24$&---&---&---&---\\   
\bottomrule
\end{tabular}
\caption{ANOVA table and effect size estimation for the CIFAR-10 test. Statistically significant effects at the $95\%$ confidence level are shown in \textbf{boldface}}
\label{tab:stats_cifar10}
\end{table}

\begin{table} [H]
\centering
\setlength{\tabcolsep}{5pt}
\begin{tabular}{lccccc}
\toprule
\multirow{2}{*}{Factor}  & \multirow{2}{*}{degrees-of-freedom} & \multirow{2}{*}{F-value}  & \multirow{2}{*}{p-value}   & Mean acc. gain ($\%$)& \multirow{2}{*}{$CI_{95}$}\\
&&&& (dynam. - const.)&\\
\midrule
Variant & $1$                  &$13.32$ & $\mathbf{1.27\times10^{-3}}$&1.14&$[0.50,1.78]$\\
Method  & $3$                  &$4.38$ &$\mathbf{0.013}$&---&---\\
Model   & $3$                  &$335.3$ &$\mathbf{1.02\times10^{-19}}$&---&---\\
Residuals  & $24$&---&---&---&---\\   
\bottomrule
\end{tabular}
\caption{ANOVA table and effect size estimation for the CIFAR-100 test. Statistically significant effects at the $95\%$ confidence level are shown in \textbf{boldface}}
\label{tab:stats_cifar100}
\end{table}

\begin{table} [H]
\centering
\setlength{\tabcolsep}{5pt}
\begin{tabular}{lccccc}
\toprule
\multirow{2}{*}{Factor}  & \multirow{2}{*}{degrees-of-freedom} & \multirow{2}{*}{F-value}  & \multirow{2}{*}{p-value}   & Mean acc. gain ($\%$)& \multirow{2}{*}{$CI_{95}$}\\
&&&& (dynam. - const.)&\\
\midrule
Variant & $1$                  &$13.72$ & $\mathbf{1.24\times10^{-3}}$&0.31&$[0.14,0.48]$\\
Method  & $2$                  &$1.41$ &0.27&---&---\\
Version & $4$                  &$1.1\times10^5$ &$\mathbf{7.76\times10^{-47}}$&---&---\\
Residuals  & $24$&---&---&---&---\\   
\bottomrule
\end{tabular}
\caption{ANOVA table and effect size estimation for the first ImageNet test (Variant $\times$ Method $\times$ Version). Statistically significant effects at the $95\%$ confidence level are shown in \textbf{boldface}}
\label{tab:stats_imagenet}
\end{table}

\begin{figure}[H]
    \centering
    \includegraphics[width=\linewidth]{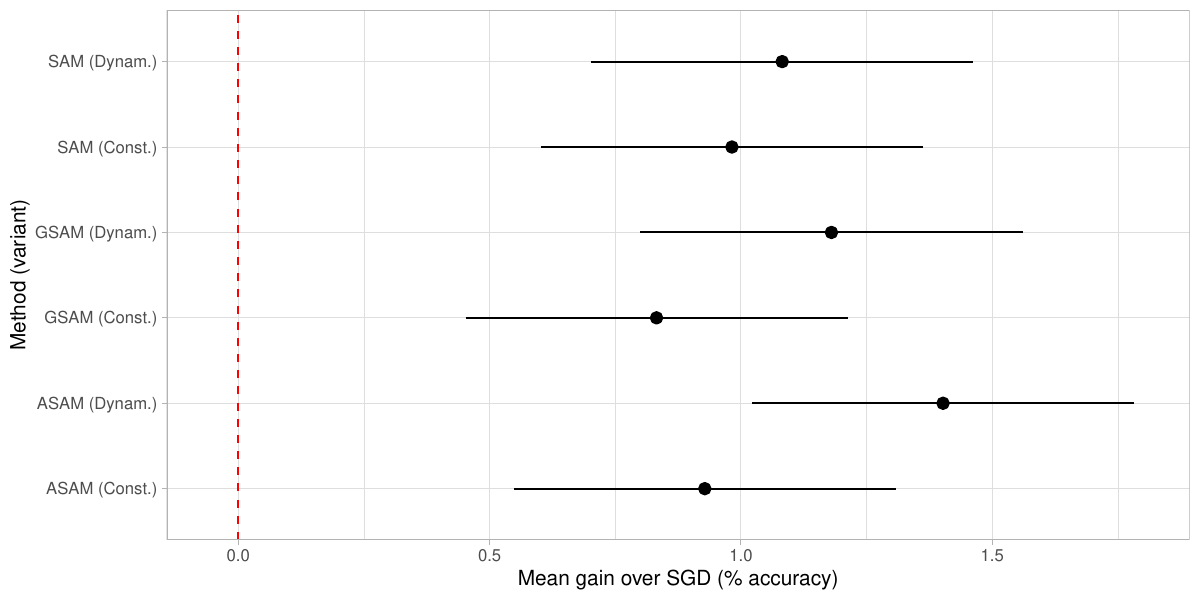}
    \caption{Mean accuracy gains of different Method-variant combinations against the SGD baseline, after removing the dataset Version effects. Using SAM-based approaches result in gains over the SGD baseline in all cases. Moreover, dynamic $\rho$ results in higher gains when compared to constant $\rho$ across all three methods tested.}
    \label{fig:imagenet2_stats}
\end{figure}

\newpage
\section{Ablation Study}
\label{sensitivity}
Figure~\ref{fig:ablation-hyperparamsmaxmin} reports the sensitivity/top-1 accuracy of CAP-SAM to its four hyperparameters, varying one at a time while holding the remaining three at their default values. Panels (a) and (b) report the initial radius $\rho_0$ and the growth factor $\gamma$; panels (c) and (d) report the bounds $\rho_{\min}$ and $\rho_{\max}$. 

CAP-SAM is therefore robust to $\rho_{\min}$ and $\rho_{\max}$, which serve only to clip the updated radius within a safe interval and have no effect unless the clip becomes active. In contrast, $\rho_0$ and $\gamma$ meaningfully influence performance and benefit from light tuning, as they act throughout training and shape how the radius changes. 

\input{ablation}

\newpage
\section{Additional Convergence Results}
Figure \ref{fig:adaptive-cifar10-vs-cifar100} shows convergence of CAP-SAM  training loss on the CIFAR-10 and CIFAR-100 datasets.
Figure \ref{fig:cifar100-train-loss} and \ref{fig:cifar100-test-acc}  shows the training loss and Test top-1 accuracy as a function of the number of epochs for CAP-SAM, SAM, and SGD on CIFAR-100 across all network architectures. CAP-SAM consistently converges faster than both SAM and SGD. 

CAP-SAM attains a consistently lower training loss than SAM while remaining above SGD, and reaches higher test accuracy than both baselines throughout training. This phenomenon clearly shows how SGD descends fastest on the training objective but generalises worst, SAM sacrifices training progress for robustness, and CAP-SAM, by perturbing only as much as the local curvature requires, is able to fit the data quickly while exceeding SAM's generalisation ability.

\input{convergence-appendix}

%% file: vector_misalignment.tex
\begin{figure}[tbp]
\centering

\begin{tikzpicture}[
    scale=1.4,
    >=Stealth,
    every node/.style={font=\small},
    signal/.style={->, line width=2pt, color=teal!70!black},
    noise/.style={->, line width=2pt, color=orange!80!black, dashed},
    sumvec/.style={->, line width=2pt, color=blue!70!black},
    epssam/.style={->, line width=2pt, color=red!90!black},
]
 
\pgfmathsetmacro{\R}{1.5}

\begin{scope}[xshift=0cm]

  \coordinate (O) at (0,0);
  \fill (O) circle (2.5pt);
  \node[below left, font=\footnotesize] at (O) {};

  \coordinate (G) at (3,0);
  \coordinate (X) at (0,3);
  \coordinate (S) at (3,3);

  \draw[signal] (O) -- (G) node[right] {$ \xi$};
  \draw[noise]  (O) -- (X) node[above] {$ \nabla f(w_t) $};
  \draw[sumvec] (O) -- (S) node[above right] {$ g_t(w_t)$};
  
  \draw[gray!35, thin, dashed] (G) -- (S);
  \draw[gray!35, thin, dashed] (X) -- (S);

  \draw[gray!55, thin] (0.22,0) -- (0.22,0.22) -- (0,0.22);

 \coordinate (eSAM) at ($(O)!{\R/veclen(3,3)}!(S)$);
  \draw[epssam] (O) -- (eSAM)
      node[pos=0.99, above,font=\small, xshift=-15pt, yshift=-5pt] {$\varepsilon_{\mathrm{SAM}}$};

 \draw[blue!55, dashed, line width=1pt]
      (\R,0) arc[start angle=0, end angle=90, radius=\R];
  \node[blue!75, font=\small] at (-0.28, \R+0.18) {$\rho$};
 
  \draw[purple!70, line width=1pt]
      ({0.55*\R}, 0)
      arc[start angle=0, end angle=45, radius=0.55*\R];
  \node[purple!80, font=\footnotesize] at
      ({0.72*\R*cos(22)},{0.72*\R*sin(22)}) {$\theta$};
 
  \node[font=\small] at (1.5, -0.65)
      {(a)};

\end{scope}

\begin{scope}[xshift=6cm]

  \coordinate (O) at (0,0);
  \fill (O) circle (2.5pt);
  \node[below left, font=\footnotesize] at (O) {};

  \coordinate (Gb)    at (3.0,  0.0);
  \coordinate (Xbvis) at (0.12, 0.45);
  \coordinate (Sb)    at ($(Gb)+(Xbvis)$);

  \coordinate (Gb) at (3, 0);
  \coordinate (Xb) at (0, 0.6);
  \coordinate (Sb) at ($(Gb)+(Xb)$);   
 
  \draw[signal] (O) -- (Gb) node[right] {$ \xi$};
  \draw[noise]  (O) -- (Xb)
      node[above, font=\footnotesize] {$ \nabla f(w_t) \!\approx\!0$};
  \draw[sumvec] (O) -- (Sb)
      node[above right, font=\footnotesize]
      {$ g_t(w_t)\approx\xi$};

  \draw[gray!35, thin, dashed] (Gb) -- (Sb);
  \draw[gray!35, thin, dashed] (Xb) -- (Sb);

  \draw[gray!55, thin] (0.22,0) -- (0.22,0.22) -- (0,0.22);

  \pgfmathsetmacro{\Sbnorm}{veclen(3,0.6)}
  \coordinate (eSAMb) at ($(O)!{\R/\Sbnorm}!(Sb)$);
  \draw[epssam] (O) -- (eSAMb)
      node[above, xshift=-23pt,yshift=-4pt, font=\footnotesize]
      {$\varepsilon_{\mathrm{SAM}}$};

  \draw[blue!55, dashed, line width=1pt]
      (\R,0) arc[start angle=0, end angle=90, radius=\R];

  \node[blue!75, font=\small] at (-0.28, \R+0.18) {$\rho$};

  \pgfmathsetmacro{\angb}{atan2(0.3,3)}
  \draw[purple!70, line width=1pt]
      (0.55*\R,0) arc[start angle=0, end angle=\angb, radius=\R];
  \node[purple!80, font=\footnotesize] at
      ({2*cos(\angb/2)},{3.2*sin(\angb/2)})
      {$\theta\!\approx\!0$};

  \node[font=\small] at (1.5, -0.65)
      {(b) };

\end{scope}

\end{tikzpicture}
\caption{ Investigation of the SAM perturbation direction near convergence. (a) Away from convergence, $\nabla f(w_t) $ dominates and the perturbation $\varepsilon_{\mathrm{SAM}}$ remains reasonably aligned with the true gradient direction. (b) Near a stationary point, $\nabla f(w_t)  \approx 0$ and $g_t(w_t) \approx \xi$, the perturbation direction becomes increasingly dominated by stochastic gradient noise making the perturbation uninformative about the local loss curvature \cite{li2024friendlysharpnessawareminimization}. }
\label{fig:vectormisalignment}

\end{figure}

%% file: ablation.tex
\pgfplotsset{
  hyperparam/.style={
    axis lines=left,
    tick align=outside,
    tick pos=left,
    grid=none,
    clip=false,
    label style={font=\scriptsize},
    tick label style={font=\scriptsize},
    xlabel style={font=\scriptsize},
    ylabel style={font=\scriptsize},
    ylabel={Top-1 Acc. (\%)},
    ymin=96.0, ymax=97.0,
    ytick={96.0,96.3,96.6,96.9},
  },
  hypline/.style={
    color={rgb,1:red,0.26;green,0.53;blue,0.76},
    line width=1.0pt,
    dash pattern=on 4pt off 2pt on 1pt off 2pt,
    mark=*,
    mark size=3pt,
    mark options={
      fill={rgb,1:red,0.26;green,0.53;blue,0.76},
      draw={rgb,1:red,0.26;green,0.53;blue,0.76},
      solid,
    },
  },
}

\begin{figure}[H]
\centering
\begin{minipage}[t]{0.48\textwidth}
\centering
\begin{tikzpicture}
\begin{axis}[
  hyperparam,
  width=\linewidth, height=5.0cm,
  xlabel={$\rho_0$},
  title={\small (a) Perturbation radius $\rho_0$},
  xmin=0.5, xmax=5.5,
  xtick={1,2,3,4,5},
  xticklabels={0.01,0.02, 0.05, 0.1, 0.2},
]
  \addplot[hypline] coordinates {
    (1, 96.42) (2, 96.45) (3, 96.60) (4, 96.66) (5, 96.29) 
  };
\end{axis}
\end{tikzpicture}
\label{fig:hyper-rho}
\end{minipage}
\hfill
\begin{minipage}[t]{0.48\textwidth}
\centering
\begin{tikzpicture}
\begin{axis}[
  hyperparam,
  title={\small (b) Growth factor $\gamma$},
  width=\linewidth, height=5.0cm,
  xlabel={$\gamma$},
  xmin=0.5, xmax=5.5,
  xtick={1,2,3,4,5},
  xticklabels={1.0001, 1.001, 1.01, 1.05, 1.1},
]
  \addplot[hypline] coordinates {
    (1, 96.84) (2, 96.66) (3, 96.47) (4, 96.58) (5, 96.08)
  };
\end{axis}
\end{tikzpicture}
\label{fig:hyper-gamma}
\end{minipage}

\vspace{1em}

\begin{minipage}[t]{0.48\textwidth}
\centering
\begin{tikzpicture}
\begin{axis}[
  hyperparam,
  title={\small (c) Minimum radius $\rho_{\min}$},
  width=\linewidth, height=5.0cm,
  xlabel={$\rho_{\min}$},
  xmin=0.5, xmax=5.5,
  xtick={1,2,3,4,5},
  xticklabels={$10^{-5}$, $10^{-4}$, $10^{-3}$, $10^{-2}$, $10^{-1}$},
]
  \addplot[hypline] coordinates {
    (1, 96.66) (2, 96.66) (3, 96.66) (4, 96.66) (5, 96.82)
  };
\end{axis}
\end{tikzpicture}
\label{fig:hyper-delta}
\end{minipage}
\hfill     
\begin{minipage}[t]{0.48\textwidth}
\centering
\begin{tikzpicture}
\begin{axis}[
  hyperparam,
  title={\small (d) Maximum radius $\rho_{\max}$},
  width=\linewidth, height=5.0cm,
  xlabel={$\rho_{\max}$},
  xmin=0.5, xmax=5.5,
  xtick={1,2,3,4,5},
  xticklabels={0.1, 0.2, 0.5, 1.0, 2.0},
]
  \addplot[hypline] coordinates {
    (1, 96.75) (2, 96.66) (3, 96.66) (4, 96.66) (5, 96.66)
  };
\end{axis}
\end{tikzpicture}
\label{fig:hyper-rhomax}
\end{minipage}
\caption{%
  Effect of hyperparameters on top-1 accuracy of
CAP-SAM with ResNet-18 on CIFAR-10.
}
\label{fig:ablation-hyperparamsmaxmin}
\end{figure}
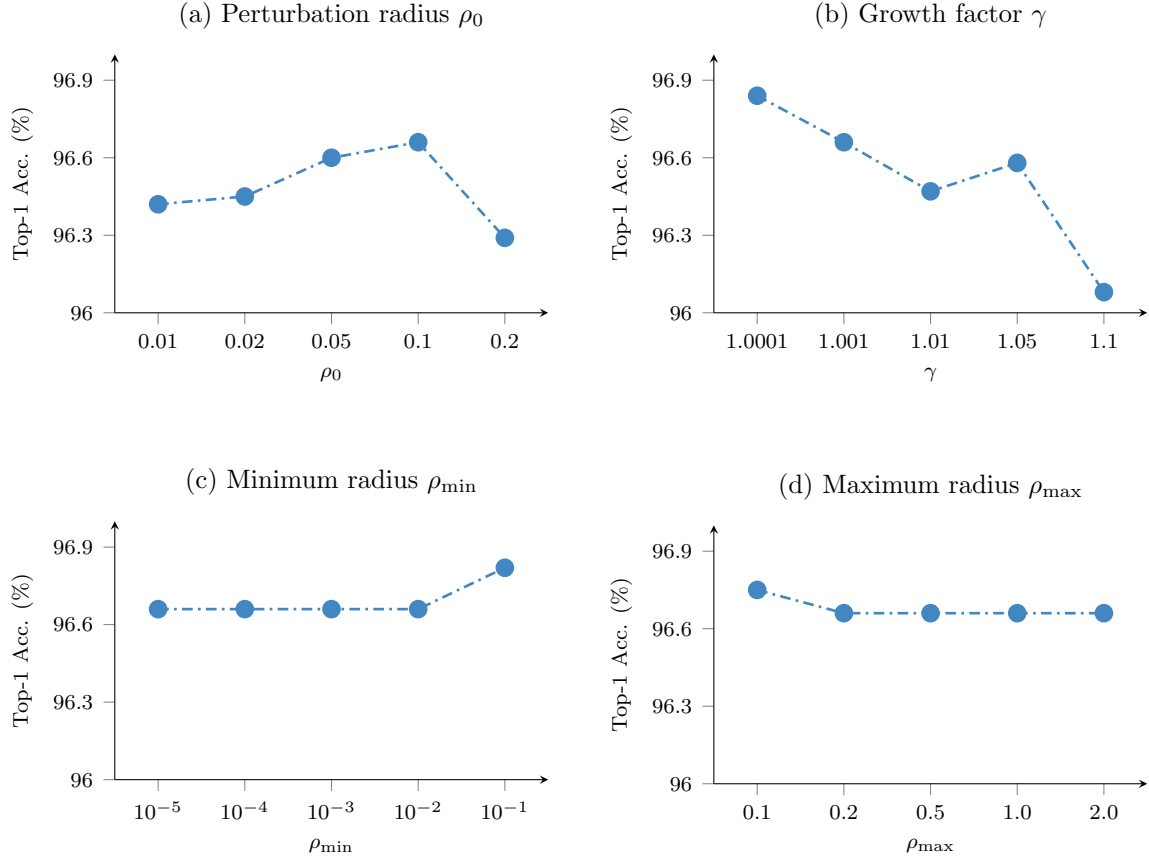

%% file: convergence-appendix.tex
\pgfplotsset{compat=1.18}

\pgfplotsset{
  sam plot/.style={
    axis lines=left,
    tick align=outside,
    tick pos=left,
    grid=major,
    grid style={dotted, gray!40},
    label style={font=\small},
    tick label style={font=\small},
    xlabel style={font=\small},
    ylabel style={font=\small},
    legend pos=north east,
    legend style={draw=none, fill=none, font=\tiny},
    clip=false,
  },
  plotadaptive/.style={
    color={rgb,1:red,0.33;green,0.51;blue,0.67},
    line width=1.2pt,
    mark=none,
  },
  plotsam/.style={
    color={rgb,1:red,0.72;green,0.40;blue,0.36},
    line width=1.2pt,
    mark=none,
  },
  sgd/.style={
    color={rgb,1:red,0.42;green,0.57;blue,0.42},
    line width=1.2pt,
    mark=none,
  },
  cifarten/.style={
    color={rgb,1:red,0.76;green,0.56;blue,0.25},
    line width=1.2pt,
    mark=none,
  },
  cifarhundred/.style={
    color={rgb,1:red,0.55;green,0.42;blue,0.60},
    line width=1.2pt,
    mark=none,
  },
}

\pagestyle{empty}

%
\begin{figure}[h]
\centering
\begin{tikzpicture}
\begin{axis}[
  sam plot,
  width=7.2cm, height=5.5cm,
  xlabel={Number of Epochs},
  ylabel={$\mathcal{L}(\mathcal{D};\,w_t)$},
  title={\small (a) ResNet-18},
  xmin=0, xmax=199,
  ymin=0, ymax=4.5,
  ytick={0,1,2,3,4},
  yticklabels={0,1,2,3,4},
]
  \addplot[cifarten] table[x=epoch, y=train_loss, col sep=comma]
    {CIFAR/CIFAR10/CIFAR10_cutout-resnet18-adaptive.csv};
  \addlegendentry{CIFAR-10}
  \addplot[cifarhundred] table[x=epoch, y=train_loss, col sep=comma]
    {CIFAR/CIFAR100/CIFAR100_cutout-resnet18-adaptive.csv};
  \addlegendentry{CIFAR-100}
\end{axis}
\end{tikzpicture}
\hfill
\begin{tikzpicture}
\begin{axis}[
  sam plot,
  width=7.2cm, height=5.5cm,
  xlabel={Number of Epochs},
  ylabel={$\mathcal{L}(\mathcal{D};\,w_t)$},
  title={\small (b) WRN-28-10},
  xmin=0, xmax=199,
  ymin=0, ymax=4.5,
  ytick={0,1,2,3,4},
  yticklabels={0,1,2,3,4},
]
  \addplot[cifarten] table[x=epoch, y=train_loss, col sep=comma]
    {CIFAR/CIFAR10/CIFAR10_cutout-wideresnet28x10-adaptive.csv};
  \addlegendentry{CIFAR-10}
  \addplot[cifarhundred] table[x=epoch, y=train_loss, col sep=comma]
    {CIFAR/CIFAR100/CIFAR100_cutout-wideresnet28x10-adaptive.csv};
  \addlegendentry{CIFAR-100}
\end{axis}
\end{tikzpicture}

\vspace{0.5em}

\begin{tikzpicture}
\begin{axis}[
  sam plot,
  width=7.2cm, height=5.5cm,
  xlabel={Number of Epochs},
  ylabel={$\mathcal{L}(\mathcal{D};\,w_t)$},
  title={\small (c) VGG-11-BN},
  xmin=0, xmax=199,
  ymin=0, ymax=4.5,
    ytick={0,1,2,3,4},
  yticklabels={0,1,2,3,4},
]
  \addplot[cifarten] table[x=epoch, y=train_loss, col sep=comma]
    {CIFAR/CIFAR10/CIFAR10_cutout-vgg11_bn-adaptive.csv};
  \addlegendentry{CIFAR-10}
  \addplot[cifarhundred] table[x=epoch, y=train_loss, col sep=comma]
    {CIFAR/CIFAR100/CIFAR100_cutout-vgg11_bn-adaptive.csv};
  \addlegendentry{CIFAR-100}
\end{axis}
\end{tikzpicture}
\hfill
\begin{tikzpicture}
\begin{axis}[
  sam plot,
  width=7.2cm, height=5.5cm,
  xlabel={Number of Epochs},
  ylabel={$\mathcal{L}(\mathcal{D};\,w_t)$},
  title={\small (d) PyramidNet-110},
  xmin=0, xmax=299,
  ymin=0, ymax=4.5,
  ytick={0,1,2,3,4},
  yticklabels={0,1,2,3,4},
]
  \addplot[cifarten] table[x=epoch, y=train_loss, col sep=comma]
    {CIFAR/CIFAR10/CIFAR10_cutout-pyramidnet110-adaptive.csv};
  \addlegendentry{CIFAR-10}
  \addplot[cifarhundred] table[x=epoch, y=train_loss, col sep=comma]
    {CIFAR/CIFAR100/CIFAR100_cutout-pyramidnet110-adaptive.csv};
  \addlegendentry{CIFAR-100}
\end{axis}
\end{tikzpicture}
\caption{%
  Training loss $\mathcal{L}(\mathcal{D};\,w_t)$ of
  CAP-SAM as a function of the number of epochs,
  comparing CIFAR-10 and CIFAR-100 across four architectures:
  (a) ResNet-18,(b)WRN-28-10,(c) VGG-11-BN and (d) PyramidNet-110.
}
\label{fig:adaptive-cifar10-vs-cifar100}
\end{figure}
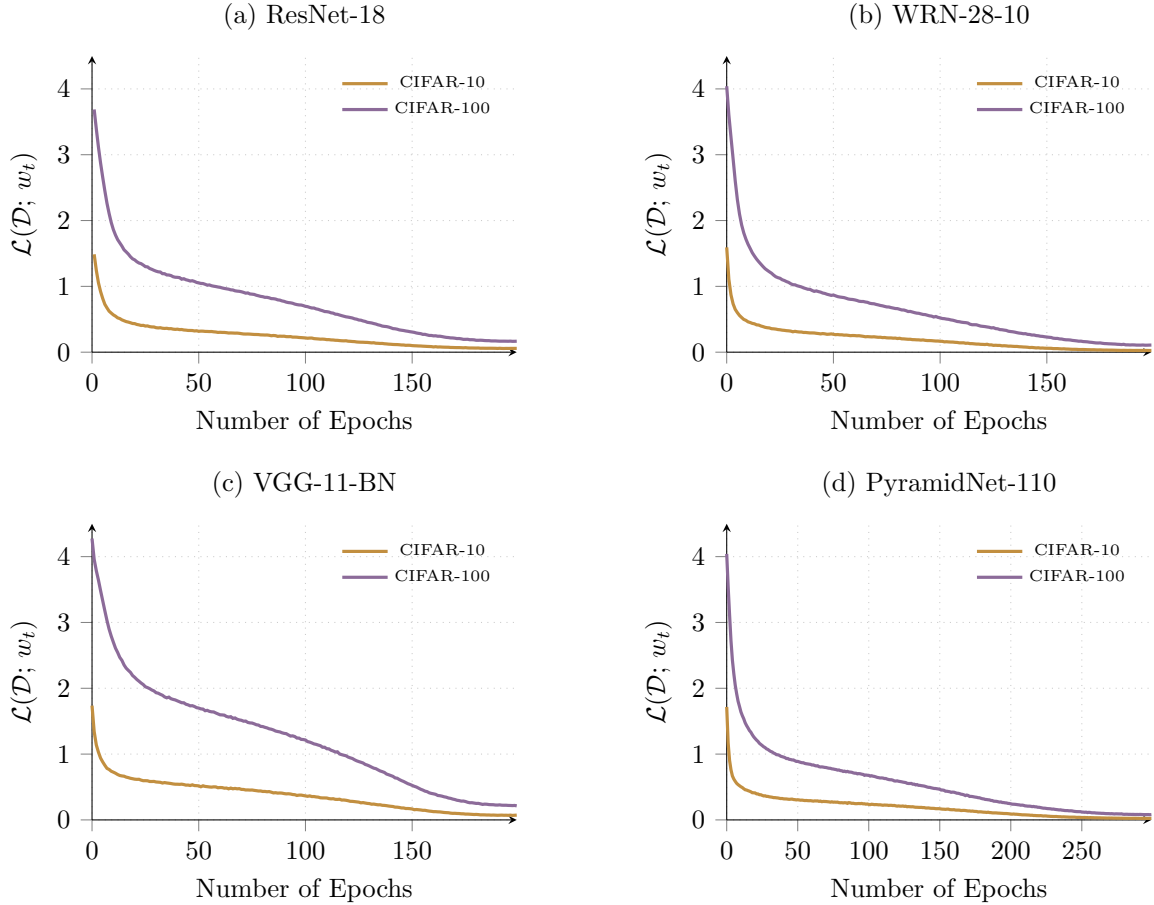

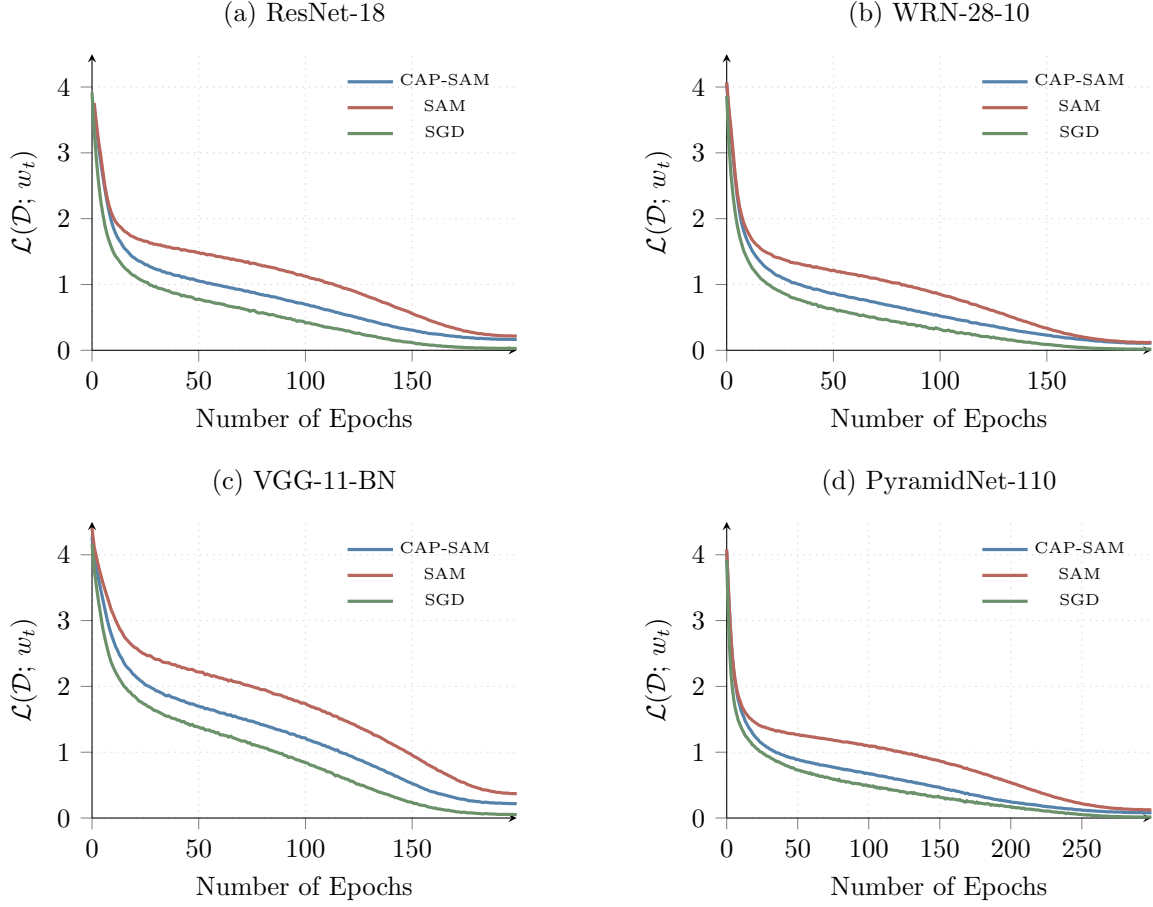
\begin{figure}[h]
\centering
\begin{tikzpicture}
\begin{axis}[
  sam plot,
  width=7.2cm, height=5.5cm,
  xlabel={Number of Epochs},
  ylabel={$\mathcal{L}(\mathcal{D};\,w_t)$},
  title={\small (a) ResNet-18},
  xmin=0, xmax=199,
  ymin=0, ymax=4.5,
    ytick={0,1,2,3,4},
  yticklabels={0,1,2,3,4},
]
  \addplot[plotadaptive] table[x=epoch, y=train_loss, col sep=comma]
    {CIFAR/CIFAR100/CIFAR100_cutout-resnet18-adaptive.csv};
  \addlegendentry{CAP-SAM}
  \addplot[plotsam] table[x=epoch, y=train_loss, col sep=comma]
    {CIFAR/CIFAR100/CIFAR100_cutout-resnet18-sam.csv};
  \addlegendentry{SAM}
  \addplot[sgd] table[x=epoch, y=train_loss, col sep=comma]
    {CIFAR/CIFAR100/CIFAR100_cutout-resnet18-sgd.csv};
  \addlegendentry{SGD}
\end{axis}
\end{tikzpicture}
\hfill
\begin{tikzpicture}
\begin{axis}[
  sam plot,
  width=7.2cm, height=5.5cm,
  xlabel={Number of Epochs},
  ylabel={$\mathcal{L}(\mathcal{D};\,w_t)$},
  title={\small (b) WRN-28-10},
  xmin=0, xmax=199,
  ymin=0, ymax=4.5,
    ytick={0,1,2,3,4},
  yticklabels={0,1,2,3,4},
]
  \addplot[plotadaptive] table[x=epoch, y=train_loss, col sep=comma]
    {CIFAR/CIFAR100/CIFAR100_cutout-wideresnet28x10-adaptive.csv};
  \addlegendentry{CAP-SAM}
  \addplot[plotsam] table[x=epoch, y=train_loss, col sep=comma]
    {CIFAR/CIFAR100/CIFAR100_cutout-wideresnet28x10-sam.csv};
  \addlegendentry{SAM}
  \addplot[sgd] table[x=epoch, y=train_loss, col sep=comma]
    {CIFAR/CIFAR100/CIFAR100_cutout-wideresnet28x10-sgd.csv};
  \addlegendentry{SGD}
\end{axis}
\end{tikzpicture}

\vspace{0.5em}

\begin{tikzpicture}
\begin{axis}[
  sam plot,
  width=7.2cm, height=5.5cm,
  xlabel={Number of Epochs},
  ylabel={$\mathcal{L}(\mathcal{D};\,w_t)$},
  title={\small (c) VGG-11-BN},
  xmin=0, xmax=199,
  ymin=0, ymax=4.5,
    ytick={0,1,2,3,4},
  yticklabels={0,1,2,3,4},
]
  \addplot[plotadaptive] table[x=epoch, y=train_loss, col sep=comma]
    {CIFAR/CIFAR100/CIFAR100_cutout-vgg11_bn-adaptive.csv};
  \addlegendentry{CAP-SAM}
  \addplot[plotsam] table[x=epoch, y=train_loss, col sep=comma]
    {CIFAR/CIFAR100/CIFAR100_cutout-vgg11_bn-sam.csv};
  \addlegendentry{SAM}
  \addplot[sgd] table[x=epoch, y=train_loss, col sep=comma]
    {CIFAR/CIFAR100/CIFAR100_cutout-vgg11_bn-sgd.csv};
  \addlegendentry{SGD}
\end{axis}
\end{tikzpicture}
\hfill
\begin{tikzpicture}
\begin{axis}[
  sam plot,
  width=7.2cm, height=5.5cm,
  xlabel={Number of Epochs},
  ylabel={$\mathcal{L}(\mathcal{D};\,w_t)$},
  title={\small (d) PyramidNet-110},
  xmin=0, xmax=299,
  ymin=0, ymax=4.5,
    ytick={0,1,2,3,4},
  yticklabels={0,1,2,3,4},
]
  \addplot[plotadaptive] table[x=epoch, y=train_loss, col sep=comma]
    {CIFAR/CIFAR100/CIFAR100_cutout-pyramidnet110-adaptive.csv};
  \addlegendentry{CAP-SAM}
  \addplot[plotsam] table[x=epoch, y=train_loss, col sep=comma]
    {CIFAR/CIFAR100/CIFAR100_cutout-pyramidnet110-sam.csv};
  \addlegendentry{SAM}
  \addplot[sgd] table[x=epoch, y=train_loss, col sep=comma]
    {CIFAR/CIFAR100/CIFAR100_cutout-pyramidnet110-sgd.csv};
  \addlegendentry{SGD}
\end{axis}
\end{tikzpicture}
\caption{%
  Training loss $\mathcal{L}(\mathcal{D};\,w_t)$ as a function of the number of
  epochs on CIFAR-100, comparing CAP-SAM, SAM, and SGD across four architectures: (a) ResNet-18, (b) WRN-28-10, (c) VGG-11-BN and (d) PyramidNet-110.
} 
\label{fig:cifar100-train-loss}
\end{figure}

\begin{figure}[h]
\centering
\begin{tikzpicture}
\begin{axis}[
  sam plot,
  width=7.2cm, height=5.5cm,
  xlabel={Number of Epochs},
  ylabel={Top-1 accuracy (\%)},
  title={\small (a) ResNet-18},
  xmin=0, xmax=199,
  ymin=0, ymax=99,
  legend style={draw=none, fill=none, font=\tiny, legend pos=south east},
]
  \addplot[plotadaptive] table[x=epoch, y=test_acc1, col sep=comma]
    {CIFAR/CIFAR100/CIFAR100_cutout-resnet18-adaptive.csv};
  \addlegendentry{CAP-SAM}
  \addplot[plotsam] table[x=epoch, y=test_acc1, col sep=comma]
    {CIFAR/CIFAR100/CIFAR100_cutout-resnet18-sam.csv};
  \addlegendentry{SAM}
  \addplot[sgd] table[x=epoch, y=test_acc1, col sep=comma]
    {CIFAR/CIFAR100/CIFAR100_cutout-resnet18-sgd.csv};
  \addlegendentry{SGD}
\end{axis}
\end{tikzpicture}
\hfill
\begin{tikzpicture}
\begin{axis}[
  sam plot,
  width=7.2cm, height=5.5cm,
  xlabel={Number of Epochs},
  ylabel={Top-1 accuracy (\%)},
  title={\small (b) WRN-28-10},
  xmin=0, xmax=199,
  ymin=0, ymax=99,
  legend style={draw=none, fill=none, font=\tiny, legend pos=south east},
]
  \addplot[plotadaptive] table[x=epoch, y=test_acc1, col sep=comma]
    {CIFAR/CIFAR100/CIFAR100_cutout-wideresnet28x10-adaptive.csv};
  \addlegendentry{CAP-SAM}
  \addplot[plotsam] table[x=epoch, y=test_acc1, col sep=comma]
    {CIFAR/CIFAR100/CIFAR100_cutout-wideresnet28x10-sam.csv};
  \addlegendentry{SAM}
  \addplot[sgd] table[x=epoch, y=test_acc1, col sep=comma]
    {CIFAR/CIFAR100/CIFAR100_cutout-wideresnet28x10-sgd.csv};
  \addlegendentry{SGD}
\end{axis}
\end{tikzpicture}

\vspace{0.5em}

\begin{tikzpicture}
\begin{axis}[
  sam plot,
  width=7.2cm, height=5.5cm,
  xlabel={Number of Epochs},
  ylabel={Top-1 accuracy (\%)},
  title={\small (c) VGG-11-BN},
  xmin=0, xmax=199,
  ymin=0, ymax=99,
  legend style={draw=none, fill=none, font=\tiny, legend pos=south east},
]
  \addplot[plotadaptive] table[x=epoch, y=test_acc1, col sep=comma]
    {CIFAR/CIFAR100/CIFAR100_cutout-vgg11_bn-adaptive.csv};
  \addlegendentry{CAP-SAM}
  \addplot[plotsam] table[x=epoch, y=test_acc1, col sep=comma]
    {CIFAR/CIFAR100/CIFAR100_cutout-vgg11_bn-sam.csv};
  \addlegendentry{SAM}
  \addplot[sgd] table[x=epoch, y=test_acc1, col sep=comma]
    {CIFAR/CIFAR100/CIFAR100_cutout-vgg11_bn-sgd.csv};
  \addlegendentry{SGD}
\end{axis}
\end{tikzpicture}
\hfill
\begin{tikzpicture}
\begin{axis}[
  sam plot,
  width=7.2cm, height=5.5cm,
  xlabel={Number of Epochs},
  ylabel={Top-1 accuracy (\%)},
  title={\small (d) PyramidNet-110},
  xmin=0, xmax=299,
  ymin=0, ymax=99,
  legend style={draw=none, fill=none, font=\tiny, legend pos=south east},
]
  \addplot[plotadaptive] table[x=epoch, y=test_acc1, col sep=comma]
    {CIFAR/CIFAR100/CIFAR100_cutout-pyramidnet110-adaptive.csv};
  \addlegendentry{CAP-SAM}
  \addplot[plotsam] table[x=epoch, y=test_acc1, col sep=comma]
    {CIFAR/CIFAR100/CIFAR100_cutout-pyramidnet110-sam.csv};
  \addlegendentry{SAM}
  \addplot[sgd] table[x=epoch, y=test_acc1, col sep=comma]
    {CIFAR/CIFAR100/CIFAR100_cutout-pyramidnet110-sgd.csv};
  \addlegendentry{SGD}
\end{axis}
\end{tikzpicture}
\caption{%
  Test top-1 accuracy as a function of the number of epochs on CIFAR-100,
  comparing CAP-SAM, SAM, and SGD across four architectures: (a) ResNet-18, (b) WRN-28-10, (c) VGG-11-BN and (d) PyramidNet-110.
}
\label{fig:cifar100-test-acc}
\end{figure}
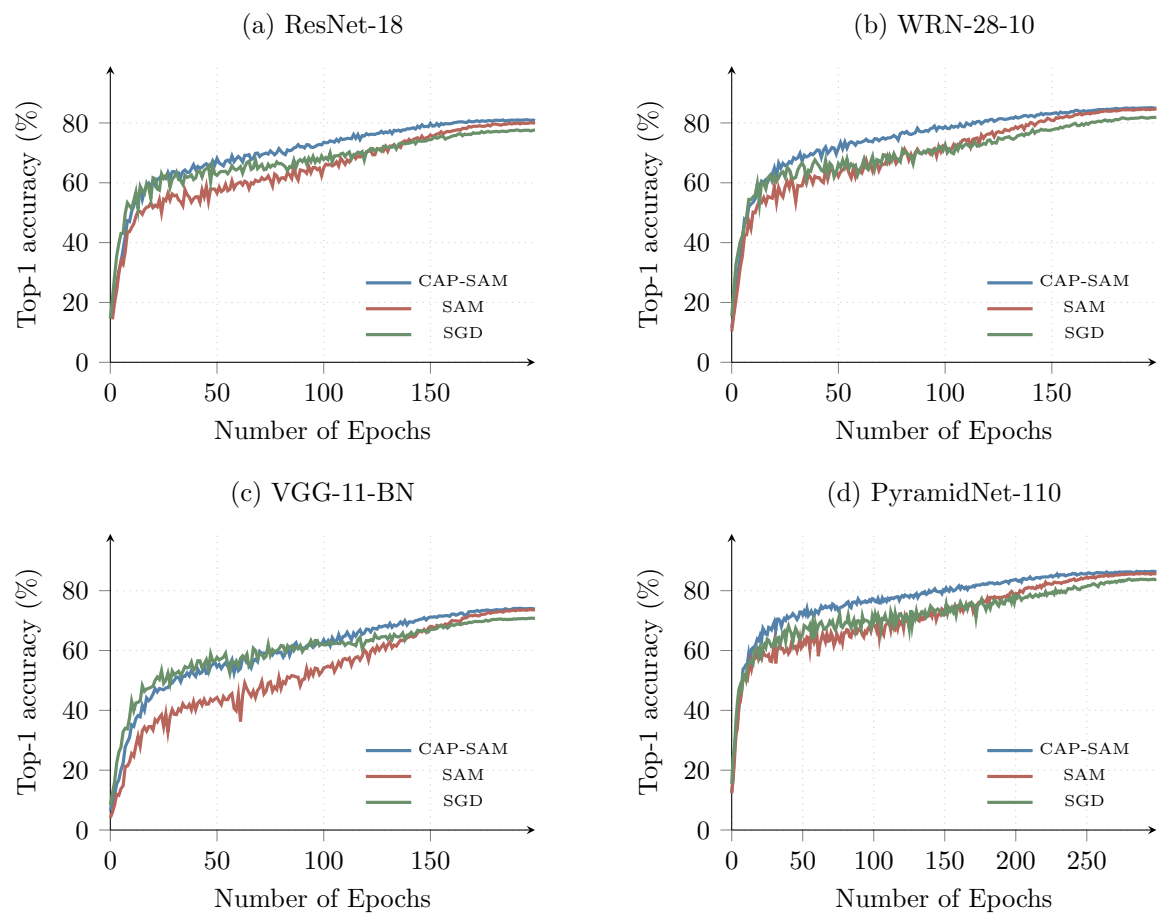